\documentclass{article}

\PassOptionsToPackage{numbers}{natbib}

\usepackage[final]{neurips_2024}

\usepackage{microtype}
\usepackage{graphicx}
\usepackage{multirow}
\usepackage{enumitem}
\usepackage[normalem]{ulem}

\usepackage[utf8]{inputenc} %
\usepackage[T1]{fontenc}    %
\usepackage{hyperref}       %
\usepackage{url}            %
\usepackage{booktabs}       %
\usepackage{amsfonts}       %
\usepackage{nicefrac}       %
\usepackage{microtype}      %
\usepackage{xcolor}         %
\usepackage[textsize=tiny]{todonotes}
\usepackage{wrapfig} %
\usepackage{epstopdf}

\usepackage{tcolorbox}
\usepackage{xfrac}
\usepackage{wrapfig}

\newtcolorbox{mybox}{
    colback=gray!10!white,
    colframe=gray!10!white,
    left=1mm,
    top=1mm,
    right=1mm,
    boxsep=0mm,
    width=14cm,  
    before=\par\smallskip\centering,
    after=\par,
    height=1cm
}

\usepackage{amsmath,amsfonts,bm,dsfont,amsthm}

\def\eqref#1{equation~\ref{#1}}

\def\1{\bm{1}}

\DeclareMathAlphabet{\mathsfit}{\encodingdefault}{\sfdefault}{m}{sl}
\SetMathAlphabet{\mathsfit}{bold}{\encodingdefault}{\sfdefault}{bx}{n}

\DeclareMathOperator*{\argmax}{arg\,max}
\DeclareMathOperator*{\argmin}{arg\,min}

\usepackage{tcolorbox}
\definecolor{grey}{rgb}{0.33, 0.33, 0.33}

\newcommand{\squishlist}{
\begin{list}{{{\small{$\bullet$}}}}
{\setlength{\itemsep}{1pt}      \setlength{\parsep}{0pt}
\setlength{\topsep}{-2pt}       \setlength{\partopsep}{0pt}
\setlength{\leftmargin}{1em} \setlength{\labelwidth}{1em}
\setlength{\labelsep}{0.5em} } }
\newcommand{\squishend}{  \end{list}  }

\makeatletter
\renewcommand*\env@matrix[1][*\c@MaxMatrixCols c]{%
  \hskip -\arraycolsep
  \let\@ifnextchar\new@ifnextchar
  \array{#1}}
\makeatother

\newcommand{\val}{\circ}

\usepackage{algorithm}
\usepackage{algorithmic} 
\newtheorem{assumption}{Assumption}[section]
\newtheorem{definition}{Definition}[section]

\newtheorem{proposition}{Proposition}[section]

\newtheorem{theorem}{Theorem}[section]

\newtheorem{lemma}{Lemma}[section]
\newtheorem{remark}{Remark}[section]

\newtheoremstyle{citing}%
  {3pt}%
  {3pt}%
  {\itshape}%
  {}%
  {\bfseries}%
  {.}%
  {.5em}%
  {\thmnote{#3}}%

\theoremstyle{citing}
\newtheorem*{varthm}{}%

\usepackage{caption}
\usepackage{subcaption}

\ifodd 1
\newcommand{\yl}[1]{\textbf{\color{red}(Yang: #1)}}
\newcommand{\zzw}[2]{\textbf{\color{blue}(Zhaowei: #1)}{\color{blue}~#2}}
\newcommand{\jlp}[1]{\textbf{\color{blue}(Jinlong: #1)}}

\newcommand{\clar}[1]{\textbf{\color{green}(NEED CLARIFICATION: #1)}}

\else

\newcommand{\com}[1]{}
\newcommand{\clar}[1]{}
\newcommand{\response}[1]{}
\newcommand{\yl}[1]{}
\newcommand{\zzw}[2]{}
\newcommand{\jlp}[1]{}

\fi

\definecolor{mygray}{gray}{0.6}

\newcommand{\infl}[1]{ \textsf{Infl}_\textsf{#1} }

\definecolor{myorange}{RGB}{241,167,108}
\definecolor{myblue}{RGB}{94,167,243}
\definecolor{mygreen}{RGB}{92,199,60}

\usepackage{colortbl}
\definecolor{best}{HTML}{BAFFCD}
\definecolor{issue}{HTML}{FFC8BA}
\definecolor{bad}{HTML}{FFC8BA}

\title{Fairness Without Harm:\\
An Influence-Guided Active Sampling Approach}

\author{
 Jinlong Pang \\
  UC Santa Cruz\\
 \texttt{jpang14@ucsc.edu} \\
  \And 
  Jialu Wang \\
  UC Santa Cruz \\
    \texttt{faldict@ucsc.edu} \\
  \And
  Zhaowei Zhu \\
  Docta.ai \\
  \texttt{zzw@docta.ai} \\
      \And     
  Yuanshun Yao \\
  Meta GenAI \\
  \texttt{ kevinyao@meta.com
} \\
  \And
    Chen Qian \\
  UC Santa Cruz \\
  \texttt{cqian12@ucsc.edu} \\
  \And
  Yang Liu\thanks{Corresponding Author: Yang Liu} \\
  UC Santa Cruz \\
  \texttt{yangliu@ucsc.edu} \\
}

\begin{document}

\maketitle

\begin{abstract}
The pursuit of fairness in machine learning (ML), ensuring that the models do not exhibit biases toward protected demographic groups, typically results in a compromise scenario. This compromise can be explained by a Pareto frontier where given certain resources (e.g., data), reducing the fairness violations often comes at the cost of lowering the model accuracy. 
In this work, we aim to train models that mitigate group fairness disparity without causing harm to model accuracy.
Intuitively, acquiring more data is a natural and promising approach to achieve this goal by reaching a better Pareto frontier of the fairness-accuracy tradeoff.
The current data acquisition methods, such as fair active learning approaches, typically require annotating sensitive attributes. However, these sensitive attribute annotations should be protected due to privacy and safety concerns.
In this paper, we propose a tractable active data sampling algorithm that does not rely on training group annotations, instead only requiring group annotations on a small validation set. Specifically, the algorithm first scores each new example by its influence on fairness and accuracy evaluated on the validation dataset, and then selects a certain number of examples for training. 
We theoretically analyze how acquiring more data can improve fairness without causing harm, and validate the possibility of our sampling approach in the context of risk disparity. We also provide the upper bound of generalization error and risk disparity as well as the corresponding connections.
Extensive experiments on real-world data demonstrate the effectiveness of our proposed algorithm. 
Our code is available at \href{https://github.com/UCSC-REAL/FairnessWithoutHarm}{\texttt{github.com/UCSC-REAL/FairnessWithoutHarm}}.

\end{abstract}

\section{Introduction}\label{sec:intro}

Machine Learning (ML) has dramatically impacted numerous optimization and decision-making processes across various domains, such as credit scoring \citep{siddiqi2012credit} and demand forecasting \citep{Carbonneau2008ApplicationOM}.
Algorithmic fairness embraces the principle, often enforced by law and regulations, that the decision-maker should not exhibit biases toward protected group membership \citep{zhu2023weak}, identified by characteristics such as race, gender, or disability.
However, the pursuit of fairness unavoidably results in a compromise scenario where reducing the fairness violations usually leads to a degradation in accuracy, which has been observed and verified by numerous literature \citep{menon2018cost,dutta2020there,zhao2022inherent,wang2022understanding,zhu2021rich}. 
Theoretically,  the phenomenon can be understood through a Pareto frontier on the tradeoff between group fairness and accuracy \citep{wang2021understanding, martinez2020minimax, balashankar2019fair,wei2022fairness}.
That is, as illustrated in Figure~\ref{fig:main_idea}, given certain resources such as training data,  when a model has reached a point on the Pareto frontier, without more data resources, it is impossible that one can improve fairness without worsening off model accuracy.

One major source of unfairness and a major cause of the fairness-accuracy tradeoff is biased training data.
If an unbiased and ``fairer'' dataset is available, we will be hopeful that unfairness can be alleviated without compromising accuracy. Furthermore,
such a ``fairer'' dataset would allow for obtaining a fair and accurate model through the standard empirical risk minimization (ERM) with cross-entropy (CE) loss.
The above observation points to a promising way to improve fairness via actively acquiring more informative data, aiming to shift towards a better Pareto frontier of the fairness-accuracy trade-off \citep{wiki_2024ProductionPossibility, lipsey1975introduction}.
However, existing approaches that seek more data, such as fair active learning \citep{anahideh2022fair}, typically require annotating sensitive attributes for training data.
In practice, these sensitive attribute information such as race and gender, should be protected due to privacy regulations \citep{holstein2019improving, andrus2021we, veale2017fairer}. 
In the normal active learning scenario, collecting more data with sensitive attributes heightens privacy and safety risks due to the increased probability of leaking sensitive information.

\begin{wrapfigure}{r}{0.5\textwidth}
\centering
    \vspace{-0.27in}
    \includegraphics[width=0.5\textwidth]{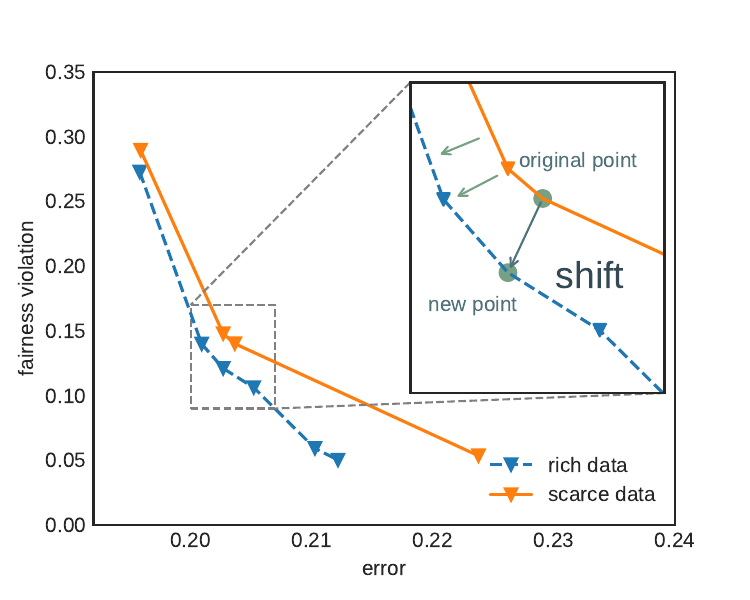}
    \caption{We compare the Pareto frontiers between the model trained with scarce data and that trained with rich data. 
    Acquiring more data is capable of shifting the Pareto frontier toward lower disparity and lower error rates. %
    In consequence, we can reach a new trade-off point that offers improved fairness and accuracy simultaneously, surpassing the original trade-off point.
    }
    \label{fig:main_idea}
\end{wrapfigure}


Therefore, we ask the following question:
\textit{When not disclosing more annotations of training sensitive attributes, how can we acquire more data to improve model fairness without sacrificing accuracy?}

In this paper, we propose a tractable active data sampling algorithm in a \textit{training sensitive attributes-free} way, which solely requires sensitive attributes on a \textit{small validation set}.
Particularly, the algorithm evaluates each example's influence on fairness and accuracy using the validation dataset for ranking and then selects a certain number of examples to supplement the training set for training.
We name our solution Fair Influential Sampling (FIS).
The core challenge is approximating the corresponding influences of each new example without accessing its sensitive attributes. Technically, we evaluate the importance (influences) of each new example by comparing its gradient to that derived from the entire validation set. This comparison helps quantify the hypothesized change of group fairness disparity metric when adding this example to the training set. As a result, the requirement of training sensitive attributes can be relaxed, as gradient derivation serves as a role of fairness constraints to measure the group fairness disparity. 
The main contributions of our work are summarized as follows.
\squishlist
\item 
We develop a tractable active data sampling algorithm (Algorithm \ref{alg:fair_influential_sampling}) that does not rely on training sensitive attributes. The algorithm scores each new example based on the combined influences of prediction and fairness and then opts for a certain number of examples for training. [Section \ref{sec:method}]
\item We theoretically analyze how acquiring more data can improve fairness without harm from a distribution shift perspective view, and validate the possibility of our sampling approach in the context of risk disparity. 
 We also provide the upper bound of generalization error and risk disparity as well as the corresponding connections (Theorem \ref{lemma:accuracy_bound} and Theorem \ref{thm:performance_gap_test}). [Section \ref{sec:theoretical_analysis}]
\item Empirical experiments on real-world datasets (CelebA, Adult, and COMPAS) substantiate our claims, indicating the effectiveness and potential of our proposed algorithm in achieving fairness for ML classifiers. [Section \ref{sec:exp}]
\squishend

\section{Related work\label{sec:related_work}}
\paragraph{Fairness-accuracy tradeoff}  There are numerous works that
have been successful at mitigating fairness disparities \citep{donini2018empirical,hardt2016equality,agarwal2018reductions,zafar2017fairness,wang2022understanding}. However, these works typically rely on protected sensitive attributes of training examples to measure the fairness disparities across groups. Moreover, a fairness-accuracy tradeoff has been shown, meaning that enforcing fair constraints heavily degrades the model performance \citep{menon2018cost,dutta2020there,zhao2022inherent}. Notably, Chen et al. \citep{chen2022fairness} characterized the change of the fairness violation when the data distribution is shifted.
Except for training sensitive attributes, this paper does not work in the classical regime of the fairness-accuracy tradeoff. By properly collecting new data, we can improve both accuracy and fairness, which cannot be achieved by working on a static training dataset that naturally incurs such a tradeoff.
Besides, compared to prior works \citep{menon2018cost, prost2019toward}, our method does not require additional assumptions about the classifier and the characteristics of the training/testing datasets (e.g., distribution shifts). 
Relevant work \citep{li2022achieving} utilizes the influence function to reweight the data examples but requires re-training. 
Our method focuses on soliciting additional samples from an external dataset while \citep{li2022achieving} reweights the existing and fixed training dataset.

\paragraph{Active learning}
The core idea of active learning is to rank unlabeled instances by developing specific measures,
including uncertainty \citep{lewis1994heterogeneous, lindenbaum2004selective}, representativeness \citep{dasgupta2008hierarchical}, inconsistency \citep{wang2012inconsistency}, variance \citep{hoi2006large}, and error \citep{roy2001toward}. 
A related line of work \citep{liu2021influence, wang2022boosting, hammoudeh2022training, xie2023active} concentrates on ranking unlabeled instances based on the influence function. Compared to these studies with a focus on prediction performance, our work poses a distinct challenge taking into account fairness violations. 
Our approach is more closely with the \textit{fair active learning} approach \citep{anahideh2022fair}. However, this framework still relies on training sensitive attributes and then unavoidably encounters the tradeoff between fairness and accuracy.

\paragraph{Fair classifiers without demographics}
There are various studies to achieve fairness without demographics. For example,
Zhao et al. \citep{zhao2022towards} explores the correlations between sensitive attributes and non-sensitive attributes to learn fair and accurate classifiers.
Yan et al. \citep{yan2020fair} investigates the class imbalance problem with a  KNN-based pre-processing method.
Chai et al. \citep{chai2022fairness} utilizes the soft labels from an overfitting teacher model to train a student model to avoid using demographics.
A line of research establishes theoretical connections between features and attributes to avoid using demographic information, employing methods like causal graphs \citep{singh2021fairness}, correlation shifts \citep{roh2023improving}, and demographic shifts \citep{giguere2022fairness}. In contrast, our approach refrains from making assumptions.
Another line of work utilizes distributionally robust optimization (DRO) to reduce fairness disparity without relying on training sensitive attributes \citep{hashimoto2018fairness, kirichenko2022last, liu2021just, lahoti2020fairness, veldanda2023hyper, sohoni2020no}. 
Although these works evaluate the worst-case group performance in the context of fairness, their approaches differ as they do not strive to equalize the loss across groups. Besides, in these studies, accuracy and worst-case accuracy are used as fairness metrics to  showcase the efficacy of the proposed algorithms. 
However, these fairness metrics are restrictive and inconsistent with common definitions such as demographic parity (DP).

\paragraph{Fair classification}
The fairness-aware learning algorithms, in general, can be categorized into pre-processing, in-processing, and post-processing methods. Pre-processing methods typically reweigh or distort the data examples to mitigate the identified biases \citep{asudeh2019assessing,azzalini2021fair,tae2021slice, sharma2020data,celis2020data,chawla2002smote,zemel2013learning,chen2018my,10.1145/3442188.3445915}.
More relevant to this work is the \textit{importance reweighting}, which assigns weights to training examples \citep{kamiran2012data, jiang2020identifying, diesendruck2020importance,choi2020fair, qraitem2023bias, li2019repair}. %
Our algorithm bears similarity to a specific case of importance reweighting, particularly the 0-1 reweighting applied to newly added data. Other parallel studies utilize importance weighting to learn a complex generative model in a weakly supervised setting \citep{diesendruck2020importance, choi2020fair}, or to mitigate representation bias in training datasets \citep{li2019repair}. 
Post-processing methods typically enforce
fairness on a learned model through calibration \citep{feldman2015computational, feldman2015certifying, hardt2016equality}, but these work might not achieve the best fairness-accuracy tradeoff \citep{woodworth2017learning,pleiss2017fairness}. 
In contrast, these post-processing works still require sensitive attributes during the inference phase. 
Recent work \citep{chen2023post} develops a \textit{bias score} classifier that operates independently of sensitive attributes; however, it is constrained to binary classifications.

\section{Preliminaries}\label{sec:preliminary}
\paragraph{Problem setup}
We consider a standard $K$-class classification task whose training (test) data distribution is $\mathcal{P}$ ($\mathcal{Q}$). Let $P:=\{z_n\}_{n=1}^{|P|}$ represent the \textit{training dataset} following distribution $\mathcal{P}$, where $|P|$ denotes the corresponding sample size. 
Each example, denoted as \underline{$z_n := (x_n, y_n)$}, comprises two random variables: the \textit{feature vector} $x$ and the \textit{label} $y$.
The model trained on $P$ is evaluated by the \textit{test dataset} $Q:=\{z^\val\}_{n=1}^{|Q|}$, where ${(\cdot)}^{\val}$ denotes that the data follows distribution $\mathcal Q$, each example \underline{$z_n^\val := (x^\val, y^\val, s^\val)$}, and the sensitive group $s_n^\val$ often refers to characteristics such as race, gender, etc. 
To align the fairness requirements on the test set $Q$ with the model trained on $P$, a popular way is to exploit the sensitive attributes $s$ \citep{donini2018empirical,wang2021understanding} or their proxies \citep{zhu2023weak} in $P$ and use them to formulate Lagrangians during training. However, extending these approaches to the active learning 
setting would require disclosing more sensitive attributes during sampling and training \citep{anahideh2022fair}, which \textit{contradicts} our goal. To avoid disclosing more sensitive attributes, we align the fairness requirements by a small hold-out \textit{validation dataset} $Q_v:=\{z_n^\val\}_{n=1}^{|Q_v|}$ that is independent and identically distributed (IID) with the test set $Q$. We defer more technical details to Section~\ref{sec:method}.

 Following the general active learning setting \citep{anahideh2022fair}, we would acquire new examples from a large \textit{unlabeled dataset} $U:=\{x'_n\}_{n=1}^{|U|}$ within a limited labeling budget $B (\ll |U|)$ \citep{liu2021influence, wang2022boosting, hammoudeh2022training}.
Denote the solicited example by \underline{$z'_n:=(x'_n, y'_n)$}, where $y'_n$ is the ground-truth label.
Note that the protected sensitive attributes from datasets $P$ and $U$ remain undisclosed during sampling and training.
In this paper, we aim to incrementally update a model that was initially trained on $P$ using standard ERM, by incorporating newly solicited data $z_n'$, such that the model can improve fairness without worsening model accuracy.
Thus, the core challenge is efficiently determining new examples that induce a significantly better Pareto frontier. In the proceeding section \ref{sec:method}, we shall delve into how to acquire new data from unlabeled set.

\paragraph{Fairness definition}
Note that this work focuses solely on active sampling to build a fairer dataset, which is then used to train the model through standard ERM with Cross-Entropy (CE) loss. Without relying on additional assumptions about the model or training/testing dataset, an intuitive and natural approach is to analyze the expected risk. 
Therefore, we introduce the concept of \textit{risk disparity} as an intermediate-term for theoretical analysis of fairness.

\begin{definition}
[Risk disparity \citep{hashimoto2018fairness, zafar2019fairness, agarwal2019fair}] 
Define $Q_k$ as the sub-distribution of $Q$ corresponding to group $k$. Given the optimized model parameters $\mathbf{w}^{P}$ trained on set $P$, risk disparity is defined as: $\mathcal{R}_{\mathcal{Q}_k}(\mathbf{w}^{P})- \mathcal{R}_{\mathcal{Q}}(\mathbf{w}^P)$, where $ \mathcal{R}_{\mathcal{Q}}(\mathbf{w}) := \mathbb{E}_{z\sim \mathcal{Q}} [\ell(\mathbf{w}, z)]$ denotes expected risk induced on target distribution $\mathcal{Q}$. 
\label{def:fairness}
\end{definition}
Definition \ref{def:fairness} naturally quantifies the discrepancy in a trained model's performance between a specific group set  $Q_k$ and the entire test set $Q$.
That is, a model can be deemed fair if it exhibits consistent performance for a group set $Q_k$ as compared to the test dataset $Q$. 
In settings such as face or speech recognition, this fairness definition implies the necessity for all demographic groups to receive the same quality service \citep{buolamwini2018gender}.
For completeness, we also include two well-known definitions of fairness:

\begin{definition}
[Demographic Parity (\texttt{DP})] A classifier \(f\) adheres to demographic parity concerning the sensitive attribute \(s\) if: $\mathbb{E}[f(\mathbf{w}, x)] = \mathbb{E}[f(\mathbf{w}, x)|s]$.
\end{definition}

\begin{definition}
[Equalized Odds (\texttt{EOd} \citep{hardt2016equality}] 
A classifier $f$ meets the equalized odds with respect to the sensitive attribute $s$ if: $\mathbb{E}[f(\mathbf{w}, x)|y] = \mathbb{E}[f(\mathbf{w}, x)|y, s]$.
\end{definition}

Even though there may be a general incompatibility between risk disparity and popular group fairness metrics like DP and EOd, under the criteria of the proposed fairness notion, these definitions could be encouraged \citep{shui2022learning, hashimoto2018fairness}. 
More details and proof can be found in the Appendix \ref{appendix:connection_between_risk_and_EOd}.
\begin{proposition}\label{property:connections_to_DP_EOd}
(Informal) Under appropriate conditions, the risk disparity can serve as a lower bound for fairness disparities based on common fairness definitions, such as DP and EOd.
\end{proposition}

\begin{remark}
[Connections to other fairness definitions]
Definition \ref{def:fairness} targets group-level risk fairness, which has similar granularity to other fairness notions such as accuracy parity \citep{zafar2017fairness}, device-level parity \citep{li2019fair}, small accuracy loss for groups \citep{zafar2019fairness, balashankar2019fair, martinez2019fairness, hashimoto2018fairness}, and bounded group loss \citep{agarwal2019fair}. 
\end{remark}

\section{Improving fairness without harm via data influential sampling}\label{sec:method}
In this section, we first introduce how to measure the importance (influence) of each example on accuracy and fairness without using the corresponding sensitive attributes, respectively.
Then, we propose an influence-guided sampling algorithm that actively acquires new data based on the influences for further training.

\subsection{Finding influential examples}

To avoid using training sensitive attributes, our primary idea is to find newly acquired data that assists in creating a ``fairer'' dataset, which allows for training a fair and accurate model via standard ERM. Initially, we explore whether newly acquired data enhances fairness by examining the training process, where the model is typically updated using gradient descent. The change of model parameters by performing one step gradient descent on newly acquired data $z'$ is 
\begin{equation}\label{eq:change_model_parameters}
    \mathbf{w}_{t+1} = \mathbf{w}_{t} - \eta\cdot \partial_{\mathbf{w}_t} \ell(\mathbf{w}_t, z')
\end{equation}
where $\eta$ refers to the learning rate and $\ell(\cdot)$ is the training loss function. It should be noted that before we solicit the true labels
of samples $z'$, we first use proxy labels. In the following subsection \ref{subsection:algorithm}, we will present a strategy for proxy labels.
Training on $z'$ affects the model's prediction on validation data $z_n^\val$ regarding both accuracy and fairness.
If the updated model $\mathbf{w}_{t+1}$ outperforms the previous one $\mathbf{w}_{t}$ evaluated on the validation dataset in terms of fairness and accuracy, this acquired data $z'$  helps to reduce the fairness disparity without worsening accuracy.

To separately measure the accuracy and fairness performance of the updated model on the validation set, we introduce two types of loss functions: \textbf{fairness loss}  $\underline{\phi(\mathbf{w}, z_n^\val)}$ and \textbf{accuracy loss} $\underline{\ell(\mathbf{w}, z_n^\val)}$, where validation data $z_n^\val = (x_n^\val, y_n^\val, s_n^\val)$.
Note that these loss functions are developed for sampling, not for training. 
Besides, training loss function $\ell(\cdot)$ can be reused as the accuracy loss function due to the same update target. One can identify training loss $\ell(\cdot, z')$ and accuracy loss $\ell(\cdot, z_n^\val)$ based on the input data used.
Without loss of generality, we assume that $\ell(\cdot)$ and $\phi(\cdot)$ are differentiable w.r.t. $\mathbf{w}$. Here, we do not restrict the generality of the fairness loss function; it can be any smoothed version of fairness metrics such as DP or EOd. 
Following this, we develop the influence of the accuracy and fairness components for finding the samples, respectively.

\paragraph{Influence of accuracy component}
When model parameters are updated from $\mathbf{w}_t$ to $\mathbf{w}_{t+1}$ by adding a new example $z'$, the influence of model's accuracy on one validation example $z_n^\val$ is:
\begin{equation*}
    \begin{aligned}
         & \infl{acc}(z', z_n^\val; \mathbf{w}_t, \mathbf{w}_{t+1})
        :=
 \ell(\mathbf{w}_{t+1}, z_n^\val) - \ell(\mathbf{w}_{t}, z_n^\val).
    \end{aligned}
\end{equation*}
For ease of notation, we use $\infl{acc}(z', z_n^\val)$ to represent $\infl{acc}(z', z_n^\val; \mathbf{w}_t, \mathbf{w}_{t+1})$. By applying first-order Taylor expansion, we obtain the following closed-form statement:

\begin{lemma}
\label{lemma:influence_of_predictions}
    The  accuracy influence of  new example $z'$  on the validation dataset $Q_v$ is:
\begin{mybox}
\begin{equation}\label{eq:influence-of-prediction}
\begin{aligned}
 \infl{acc}(z')
:= \sum\nolimits_{n \in |Q_v|} \infl{acc}(z', z_n^\val)   
\approx   - \eta \sum\nolimits_{n \in|Q_v|} \left\langle \partial_{\mathbf{w}_{t}} \ell(\mathbf{w}_t, z'),  \partial_{\mathbf{w}_{t}} \ell(\mathbf{w}_t, z_n^\val)\right\rangle
\end{aligned} 
\end{equation}
\end{mybox}
\end{lemma} 
Intuitively, the more negative $\infl{acc}(z')$ is, the more positive the model accuracy (performance) that example $z'$ can provide.

\paragraph{Influence of fairness component}
When model parameters are updated from $\mathbf{w}_t$ to $\mathbf{w}_{t+1}$ by adding a new example $z'$, the influence of model's fainess on one validation example $z_n^\val$ is:
\begin{equation}
   \infl{fair}(z', z_n^\val; \mathbf{w}_t, \mathbf{w}_{t+1}) := \phi(\mathbf{w}_{t+1}, z_n^\val) 
- \phi(\mathbf{w}_{t},  z_n^\val) .     
\end{equation}

For simplicity, we write $\infl{fair}(z', z_n^\val; \mathbf{w}_t, \mathbf{w}_{t+1})$ as $\infl{fair}(z', z_n^\val)$. Then, similarly, we have:
\begin{lemma}
\label{lemma:influence_of_fairness}
The fairness influence of new example $z'$ on the validation dataset $Q_v$ is:
\begin{mybox}
\begin{equation}\label{eq:influence-of-fairness}
\begin{aligned}
     \infl{fair}(z') 
     : =  \sum\nolimits_{n \in |Q_v|} \infl{fair}(z', z_n^\val) 
     \approx   - \eta \sum\nolimits_{n\in|Q_v|} \left\langle \partial_{\mathbf{w}_t} \ell(\mathbf{w}_t, z'), \partial_{\mathbf{w}_t}  \phi(\mathbf{w}_t, z_n^\val) \right\rangle 
\end{aligned}
\end{equation}
\end{mybox}
\end{lemma}
Similar to the accuracy component, the greater the negativity of $\infl{fair}(z')$ is, the greater the positive impact that the example $z'$ has on fairness.

\paragraph{Intuitions}
These two components evaluate the accuracy and fairness impact of each example by comparing the gradient originating from a single data sample with the gradient derived from the entire validation set, respectively. 
This comparison helps quantify the potential advantage of including this specific example in training. For instance, if the gradient obtained from one example has a similar direction to the gradient from the validation set, it indicates that incorporating this example contributes to enhancing the model's fairness or accuracy.

\paragraph{Training sensitive attributes are not disclosed}
One can easily check that neither the influence of accuracy nor fairness components require the sensitive attributes of any example $z'$, as the example $z'$ only appears in the first-order gradient of the accuracy loss $\partial_{\mathbf{w}_t} \ell(\mathbf{w}_t, z')$. In the fairness component, calculating the  $\partial_{\mathbf{w}_t}  \phi(\mathbf{w}_t, z_n^\val)$ only relies on validation example $z^\val_n$'s sensitive attributes.
Here, we also validate how accurate the first-order estimation of the influence is in comparison to the real influence \citep{koh2017understanding}, and find that the estimated influences for most of the examples are very close to their actual influence values. We refer the readers to Appendix \ref{appendix:influence_score} for more details.

Even without disclosing training sensitive attributes, the correlations between non-demographic features and demographic information may still lead to privacy leakage issues \citep{zhao2022towards}. 
To address this potential privacy concern, we provide further discussions and theoretical analysis using differential privacy in Appendix \ref{appendix:correlation_non-demongraphic_and_demographic}.

\subsection{Algorithm: fair influential sampling (FIS)}
\label{subsection:algorithm}

Following Lemma \ref{lemma:influence_of_predictions} and Lemma \ref{lemma:influence_of_fairness}, we can efficiently select those examples with the most negative fairness influence and negative accuracy influence. This sampling method aids in reducing fairness disparities without worsening model accuracy.

\paragraph{Labeling}
Before presenting our sampling algorithm, it is necessary to address the problem of not accessing the true labels of new solicited examples.  
Lacking the label information for new examples poses a challenge in determining the corresponding influence on accuracy and fairness, a fact that is substantiated by Lemma \ref{lemma:influence_of_predictions} and Lemma \ref{lemma:influence_of_fairness}.
Intuitively, one can always recruit human annotators to get the ground-truth labels for those unlabeled examples. However, it is impractical due to the limited labeling budgets.
To tackle this problem, another common approach is utilizing a model that has been effectively trained on dataset $P$ to produce proxy labels, which approximate the calculation of influences for examples from a substantial unlabeled dataset $U$. It's important to note that these proxy labels are exclusively used during the sampling phase. To maintain good model performance, we still need to inquire about the true labels of the selected data examples for subsequent training. 
Here, we propose to annotate the proxy labels with the model trained on the labeled set $P$. In particular, we introduce a strategy that employs lowest-influence labels for annotating label $\hat y'$ given $x'$:
\begin{equation}\label{eq:strategy}
    \hat y' =  \argmin_{k\in\{1,\cdots, K\}} ~ |\infl{acc}(x', k) |,
\end{equation}
Here, we denote $\hat z':=(x',\hat y')$ for the proxy labels.

\begin{algorithm}[t]
\caption{Fair influential sampling (FIS)}\label{alg:fair_influential_sampling} 
\begin{algorithmic}[1]
\STATE \textbf{Input:} training set $P$, unlabeled set $U$, validation set $Q_v$, new acquired set $S_t =\{\}, \forall t \in [T]$ rounds, number of new selected examples in each round $r$, tolerance $\epsilon$.   \\
\STATE \textbf{Warmup:} Train a classifier $f$ solely on $P$ by minimizing the empirical risk $R_{\ell}$. Obtain model parameters $\mathbf{w}_1$ and validation accuracy (on $Q_v$) $\text{VAL}_0$. \\
\FOR {$t$ \textbf{in} $\{1, 2, \cdots, T\}$}
\STATE Guess proxy label $\hat y'$ for new examples $\hat z'$ using Eq. (\ref{eq:strategy}).
\STATE Compute the influence of accuracy and fairness component using Eq. (\ref{eq:influence-of-prediction}) and Eq. (\ref{eq:influence-of-fairness}):
\STATE \( \qquad \qquad \qquad
    S_\text{original} = \{ \infl{fair}(\hat z')  ~|~  \infl{acc}(\hat z') \le 0, \infl{fair}(\hat z') \le 0, \hat z' \in U\}
\)
\WHILE{$|S_t| < r$}  
\STATE Find top-$(r-|S_t|)$ annotated examples $\hat{z}'_n$ based on the lowest fairness influence and then inquire about true labels $y'$: 
\STATE \(\qquad \qquad \qquad \qquad \{z'_n\} \stackrel{\text{inquiring}}{\longleftarrow} \{\hat{z}'_n\} \leftarrow \textsf{Top-$(r-|S_t|)$} (S_\text{original})\)
\STATE  $S_t \leftarrow S_t \cup \{z'_n ~|~ \infl{acc}(z'_n) \le 0, \infl{fair}(z'_n) \le 0\} $
\STATE $U \leftarrow U \cap S_t; \quad S_\text{original} \leftarrow S_\text{original} \cap S_t $
\ENDWHILE \\
\STATE Continue to train the model $f$ on the set $S_t$ via standard ERM. Obtain the updated model parameters $\mathbf{w}_{t+1}$. If the model's validation accuracy (on $Q_v$) $\text{VAL}_{t}$ does not meet the desired threshold $\text{VAL}_0$, reject the updated model.
\ENDFOR

\STATE \textbf{Output:} models $\{\mathbf{w}_t ~|~ \text{VAL}_t > \text{VAL}_0 - \epsilon \}$
\end{algorithmic}
\end{algorithm}

\paragraph{Proposed algorithm}
The full procedure is outlined in Algorithm \ref{alg:fair_influential_sampling}. 
Note that the tolerance $\epsilon$ is applied to monitor the performance drop in validation accuracy.
In Line 2, we initiate the process by training a classifier $f$ solely on dataset $P$, that is, performing a warm start. Subsequently, $T$-round sampling iterations are applied to acquire more examples to dataset $P$.
Following the iterative fashion, FIS guesses labels using Eq. (\ref{eq:strategy}) in Line 4. 
Then, we calculate the scores for proxy examples based on the accuracy and fairness influence using Eq. (\ref{eq:influence-of-prediction}) and Eq. (\ref{eq:influence-of-fairness}) respectively. 
In Lines 6-12,  we would opt for $r$ samples based on the influence scores, and inquire about the true labels of these examples.
However, due to the gap between proxy labels $\hat y'$ and true label $y$, the accuracy and fairness influences of the top-$r$ samples based on inquired true labels $z'$ may not necessarily satisfy the same conditions ($\infl{acc}(\hat z') \le 0, \infl{fair}(\hat z') \le 0$).
Therefore, we use a while loop to iteratively select the top-$(r-|S_t|)$ examples for labeling until we obtain $r$ samples whose fairness influences based on true labels meet the conditions.
Subsequently, in Line 11, we update sets $U$ and $S_{\text{original}}$ to prevent duplicate sampling.
In Line 13, we would continue training using new examples with true inquired labels from set $S_t$. 
We save the model parameters at each round as checkpoints $\mathbf{w}_t$. 
To avoid potential accuracy drops incurred by excessively large random perturbations, we exclusively choose and offer models for output whose validation accuracy exceeds the initial validation accuracy $\text{VAL}_0$. Although we propose a specific strategy for guessing labels, our algorithm is flexible and compatible with other labeling methods. A comparative analysis of computational costs is detailed in Appendix \ref{apx:runtime_complexity}.

\section{How more data improves fairness without harm?}\label{sec:theoretical_analysis}
In general, acquiring new data to supplement the original training dataset would potentially raise the distribution shift problem, affecting both accuracy and fairness.
In this section, from a distribution shift perspective view, we first present a generalization error bound (accuracy side, Theorem \ref{lemma:accuracy_bound}) and risk disparity bound (fairness side, Theorem \ref{thm:performance_gap_test}).
The theoretical results jointly provide a high-level key insight that controlling the negative impact of distribution shift on generalization error, which refers to the model accuracy, could allow for improving fairness without harm.
This theoretical insight validates the possibility of our sampling approach.

Without loss of generality, we discretize the whole distribution space and suppose that the train/test distributions are both drawn from a series of component distributions $\{\pi_1, \cdots, \pi_I\}$ \citep{feldman2020does}. %
Then, the empirical risk $\mathcal{R}_{P}(\mathbf{w})$ calculated over a training set $P$ can be reformulated by splitting samples based on the component distributions:
\begin{equation*}
\begin{aligned}
         \mathcal{R}_{P}(\mathbf{w}) 
        :=   \mathbb{E}_{z \in P}[\ell(\mathbf{w}, z)] 
          =  \sum\nolimits_{i=1}^{I}p^{(P)}(\pi=i) \cdot \mathbb{E}_{z\sim \pi_i} [\ell(\mathbf{w}, z)].
\end{aligned}
\end{equation*}
where $p^{(P)}(\pi=i)$ represents the frequencies of examples in $P$ drawn from component distribution $\pi_i$.
Then, we can define the measure of probability distance between two sets or distributions as $\text{dist}(\mathcal{A}, \mathcal{B}):= \sum_{i=1}^{I} | p^{(\mathcal{A})}(\pi=i) - p^{(\mathcal{B})}(\pi=i)|$. 
To reflect the implicit unfairness in the models, we introduce two basic assumptions in convergence analysis \citep{li2019fair}. 

\begin{assumption}
[$L$-Lipschitz Continuous] There exists a constant $L >0$,  for any $\mathbf{v}, \mathbf{w} \in \mathbb{R}^d$, $\mathcal{R}_{P}(\mathbf{v}) \leq \mathcal{R}_{P}(\mathbf{w}) + \langle \nabla \mathcal{R}_{P}(\mathbf{w}), \mathbf{v}- \mathbf{w} \rangle + \frac{L}{2}\lVert \mathbf{v}- \mathbf{w} \rVert^2_2$.
\label{assumption:lipschitz}
\end{assumption}

\begin{assumption}
[Bounded Gradient on Random Sample] The stochastic gradients on any sample $z$ are uniformly bounded, {\em i.e.},  $\mathbb{E}[||\nabla \mathcal{R}_{P}(\mathbf{w}_t, z)||^2] \leq G^2$, and training epoch $t\in [1, \cdots, T]$.
\label{assumption:bounded_gradient} %
\end{assumption}

Analogous to Assumption \ref{assumption:bounded_gradient}, we further make a mild assumption to bound the loss over the component distributions $\pi_i$ according to the corresponding model, that is, $\mathbb{E}_{z\sim \pi_i} [\ell(\mathbf{w}^P, z)] \leq G_P, \forall i \in I$, where $G_P$ is a bounding constant.  For completeness, we first analyze the upper bound of generalization error, specifically from the standpoint of distribution shifts. Omitted proof can be found in Appendix \ref{appendix:omitted_proof}.

\begin{theorem}
[Generalization error bound] Let $\text{dist}(\mathcal{P}, \mathcal{Q})$, $G_P$ be defined therein. With probability at least $1-\delta$ with $\delta \in (0,1)$, the generalization error bound of the model trained on dataset $P$ is
\begin{align*}
        \mathcal{R}_{\mathcal{Q}}(\mathbf{w}^P) \leq   \underbrace{G_P \cdot \text{dist}(\mathcal{P}, \mathcal{Q})}_{\text{ distribution shift}} + \sqrt{\frac{\log(4/\delta)}{2 |P|}} +  \mathcal{R}_{P}(\mathbf{w}^{P}).
\end{align*}
\label{lemma:accuracy_bound}
\end{theorem}
\vspace{-2ex}
Note that the generalization error bound is predominantly impacted by the shift in distribution, especially when we consider an overfitting model, i.e., the empirical risk $\mathcal{R}_{P}(\mathbf{w}^{P}) \to 0$.

\begin{theorem}
[Upper bound of risk disparity] Suppose  $\mathcal{R}_{\mathcal{Q}}(\cdot)$ follows Assumption \ref{assumption:lipschitz}. Let $\text{dist}(\mathcal{P}, \mathcal{Q})$, $G_P$, $\text{dist}(\mathcal{P}_k, \mathcal{Q}_k)$ and $\text{dist}(P_k, P)$ be defined therein. The initial learning rate $\eta_0$ satisfies $\eta_0^2 < \frac{1}{\sqrt{2}TL}$, where $T$ is the number of training epochs.
 Given model $\mathbf{w}^P$ and $\mathbf{w}^k$ trained exclusively on group $k$'s data $P_k$,  with probability at least $1-\delta$ with $\delta \in (0,1)$, then the upper bound of  risk disparity is
\begin{multline}\label{eq:performance_gap_test}
    \mathcal{R}_{\mathcal{Q}_k}(\mathbf{w}^{P})- \mathcal{R}_{\mathcal{Q}}(\mathbf{w}^P)
 \leq \underbrace{G_k \cdot \text{dist}(\mathcal{P}_k, \mathcal{Q}_k) +  G_P \cdot \text{dist}(\mathcal{P}, \mathcal{Q})}_{\text{ distribution shift}} + \underbrace{4 L^2  G^2 \cdot  \text{dist}(P_k, P)^2}_{\text{group gap}}  + \Upsilon
\end{multline}
where  
{\small $\Upsilon =\sqrt{\frac{\log(4/\delta)}{2|P|}} +  \sqrt{\frac{\log(4/\delta)}{2|P_k|}} + \varpi  + \varpi_k.$}
Note that $\mathbb{E}_{z\sim \pi_i} [\ell(\mathbf{w}^k, z)] \leq G_k$,  $\varpi = \mathcal{R}_{P}(\mathbf{w}^P) - \mathcal{R}_{\mathcal{Q}}^*(\mathbf{w}^{\mathcal{Q}})$ and $\varpi_k = \mathcal{R}_{P_k}(\mathbf{w}^{k})  - \mathcal{R}_{\mathcal{Q}_k}^*(\mathbf{w}^{\mathcal{Q}_k})$. $\varpi$ and $\varpi_k$ can be regarded as constants because  $\mathcal{R}_{P}(\mathbf{w}^P)$ and $\mathcal{R}_{P_k}(\mathbf{w}^{k}) $ correspond to the empirical risks, $\mathcal{R}_{\mathcal{Q}}^*(\mathbf{w^{\mathcal{Q}}})$ and $\mathcal{R}_{\mathcal{Q}_k}^*(\mathbf{w^{\mathcal{Q}_k}})$ represent the ideal minimal empirical risk of model $\mathbf{w}^{\mathcal{Q}}$ trained on distribution $\mathcal{Q}$ and $\mathcal{Q}_k$, respectively.
\label{thm:performance_gap_test}
\end{theorem}

\paragraph{Interpretations of Theorem~\ref{thm:performance_gap_test}} 
Eq.~(\ref{eq:performance_gap_test}) illustrates several aspects that induce unfairness. (1) \textit{Group biased data}. For group-level fairness, the more balanced the data is, the smaller the risk disparity would be;
(2) \textit{distribution shift}. For source/target distribution, the closer the distributions are, the smaller the performance gap would be; (3) \textit{Data size}. For training data size, a larger data size (potentially eliminating data bias across groups) would lead to a smaller performance gap.

\paragraph{Main observation}
Theorem~\ref{lemma:accuracy_bound} underscores how the generalization error is impacted by distribution shifts.  Theorem~\ref{thm:performance_gap_test} implies that the risk disparity is essentially influenced by the distribution shift and the inherent group gap term.
In practice, approaches that mitigate the group gap, such as imposing fairness regularizers, acquiring new data, or reweight the training data samples \citep{kamiran2012data}, will inevitably incur additional distribution shifts between the training and test data.
The incurred distribution shift further leads to a performance drop due to the generalization error in Theorem~\ref{lemma:accuracy_bound}.
Nonetheless, one theoretical insight is that if one can control the negative impacts of potential distribution shifts through generalization error while implementing fairness-enhancing strategies, it is possible to achieve the goal of improving fairness without causing harm.
This high-level insight supports the effectiveness of our proposed sampling approach, in which we acquire new data to reduce the group gap through fairness components while preventing the potential adverse impacts of distribution shifts using the accuracy influence component.

\section{Empirical results}\label{sec:exp}

In this section, we empirically demonstrate the disparate impact across groups and present the effectiveness of the proposed Fair Influential Sampling method to mitigate the disparity.

\subsection{Experimental setup}
We evaluate the performance of our algorithm on three real-world datasets across three different modalities: CelebA \citep{liu2015deep}, UCI Adult \citep{asuncion2007uci} and Compas \citep{angwin2016machine}.
We implement the fairness loss $\phi(\cdot)$ based on three common group fairness metrics: difference of demographic parity (\texttt{DP}), difference of equality of opportunity (\texttt{EOp}), and difference of equal odds (\texttt{EOd}). 
We compare our method with five baselines: 1) Base (ERM): directly train the model on the training dataset $P$; 2) Random: train the model on dataset $P$ and randomly sampled data from $U$ with inquired true labels; 3) BALD \citep{branchaud2021can}: active sampling according to the mutual information; 4) ISAL \citep{liu2021influence}: selects unlabeled examples based on the calculated influence in an active learning setting. We apply model predictions as pseudo-labels;  5) Just Train Twice (JTT) \citep{liu2021just}: reweighting those misclassified examples for re-training. Here, we examine a weight of 20 for misclassified examples, marked as JTT-20.
Recall that we present the average result of the classifier $\mathbf{w}_t$ outputs from Algorithm \ref{alg:fair_influential_sampling}. 
The general term ``\textit{fairness violation}'' is utilized to quantify the absolute differences based on fairness metrics, such as DP and EOd.
More details on datasets and hyper-parameters are provided in Appendix \ref{appendix:dataset_setting}.

\subsection{Main results}

Note that all the experimental results presented subsequently are from three independent trials, each conducted with distinct random seeds. We present the primary results as tuples in the form (\texttt{test\_accuracy}, \texttt{fairness\_violation}) to facilitate comparison of the fairness-accuracy tradeoff.
Due to space limits, we provide a full version of the experimental results (tables) in Appendix \ref{appendix:full_results}.

\paragraph{Results on image datasets}
Initially, we train a vision transformer using a patch size of $(8,8)$ on the CelebA face attribute dataset \citep{liu2015deep}. We select four binary classification targets, including \texttt{Smiling}, \texttt{Attractive}, \texttt{Young}, and \texttt{Big Nose}. The sensitive attribute is \texttt{gender}. 2\% of the labeled data is allocated for training, while the remaining 98\% is reserved for sampling purposes. Then, the test dataset is split into two independent portions: a new test set and a validation set, with 10\% of the test data randomly designated as the hold-out validation set.
For ease of computation, only the last two layers of the model are used to calculate the influence of accuracy and fairness components.
For Figure \ref{fig:all_results}, one main observation is that FIS outperforms baselines with a significant margin on three fairness metrics while maintaining the same accuracy level. This improvement, as indicated in Theorem \ref{thm:performance_gap_test}, can be attributed to FIS assigning priority to new examples based on the fairness influence, then avoiding accuracy reduction via their accuracy influence.

\begin{figure}
    \centering
    \includegraphics[width=1\linewidth]{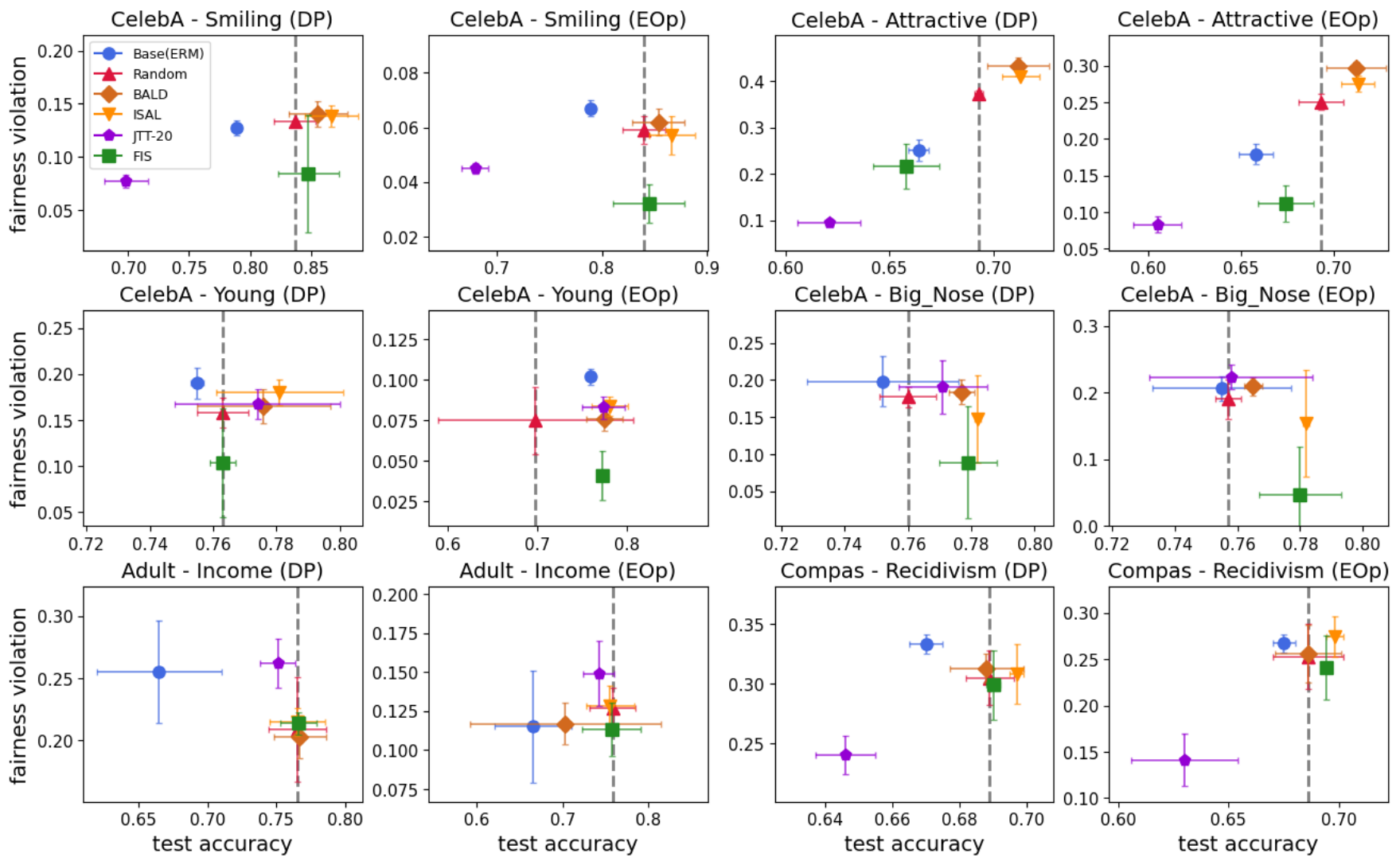}
    \caption{Main results on CelebA, Adult and Compas datasets. The Y axis shows \texttt{fairness\_violation}; X axis denotes \texttt{test\_accuracy}.
     \textbf{CelebA:} Four binary targets: \texttt{Smiling}, \texttt{Attractive}, \texttt{Young}, and \texttt{Big\_Nose}; Sensitive attribute: \texttt{gender}. 
     \textbf{Adult}: Binary target: \texttt{Income}; Sensitive attribute: \texttt{Age}.
     \textbf{Compas}: Binary target: \texttt{Recidivism}; Sensitive attribute: \texttt{Race}. We select two fairness metrics \texttt{DP} and \texttt{Eop} to measure fairness violations for each setting. The vertical dotted line at the random baseline accuracy helps easily identify which results achieve fairness without sacrificing performance (accuracy).
     }
    \label{fig:all_results}
\end{figure}

\paragraph{Results on tabular datasets}
Next, we work with multi-layer perceptron (MLP) with two layers trained on the Adult \citep{asuncion2007uci} and Compas dataset \citep{angwin2016machine}, respectively. We select \texttt{age} as the sensitive attribute for the Adult dataset and \texttt{race} for the Compas dataset.
For two datasets, we resample the data to balance the class and group membership \citep{chawla2002smote}. The whole dataset is split into training and test sets at a 4:1 ratio. Then, we randomly re-select 20\% of the training set for initial training and the remaining 80\% for sampling. Also, 20\% examples of the test set are selected to form a validation set.
The MLP model is a two-layer ReLU network with a hidden size of 64.
We utilize the whole model parameters to compute the influence of accuracy and fairness for examples.
Figure \ref{fig:all_results} summarizes the main results of the Adult and Compas datasets. On the Adult dataset, we observe that our sampling method achieves the lowest violation for equality of opportunity and has a comparable performance for the DP metric. Besides, our algorithm achieves a much better accuracy-fairness trade-off than other baselines on the Compas dataset. JTT-20 achieves a lower fairness violation with the price of a significant accuracy drop compared to other baselines.

\subsection{Ablation study}
\paragraph{What is the impact of label budgets?} Here, we examine how the varying label budgets $r$ affect the trade-off between accuracy and fairness. For ease of comparison, we adhere to a consistent label budget per round to illustrate their respective impacts. As shown in Figure \ref{table:validation_size_celeba}, our method consistently preserves a lower fairness violation than the BALD and ISAL baselines with a similar test accuracy. While we observe that the JTT-20 algorithm can achieve a near-zero fairness violation under a limited budget on the CelebA dataset, we argue that the model accuracy is rather uninformative (about 50\%).
More empirical results can be found in Appendix \ref{appendix:impact_of_label_budget}.

\paragraph{How does the validation set size affect the performance?} We further explore the impact of adjusting the validation set size on our algorithm's performance. We present the test accuracy and fairness violations across different validation set sizes on the CelebA dataset. Note that the default validation set size is set to 1\% of the whole dataset size.
In particular, the minimum scale of the validation set size is set to $\sfrac{1}{5}\times$ (nearly 400 CelebA images). The results in Figure~\ref{table:validation_size_celeba} indicate that our algorithm still retains the test accuracy and fairness violation when we vary the validation set size.
Additional results conducted on Adult and Compas datasets are provided in Appendix \ref{appendix:impact_of_validaion_size}.

\begin{figure}[htp]
    \centering
    \begin{minipage}{0.49\textwidth}
        \centering
        \includegraphics[width=\linewidth]{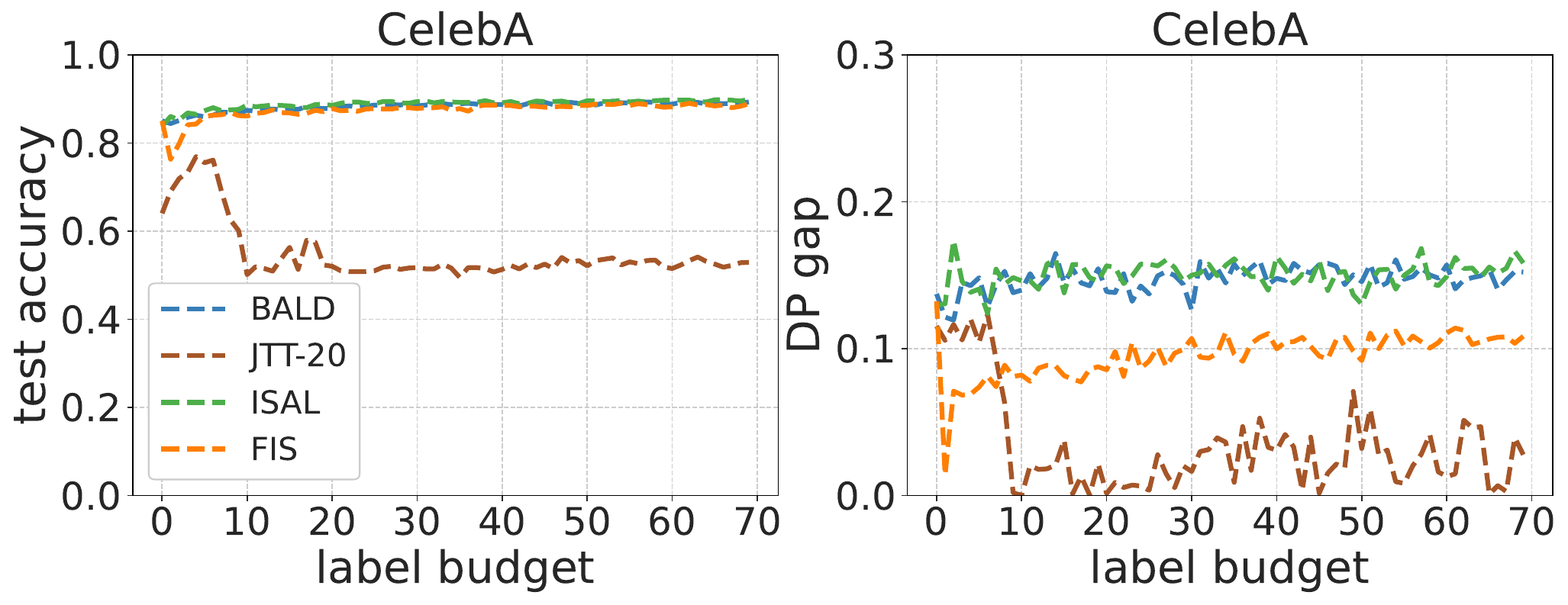}
    \end{minipage}%
    \hspace{1ex}
    \begin{minipage}{0.49\textwidth}
        \centering
        \resizebox{\linewidth}{!}{
        \begin{tabular}{c ccc}
        \toprule
            ~ &   \multicolumn{3}{c}{\textbf{CelebA - Smiling}} \\
            \cmidrule{2-4}
            ~ & (Test\_acc$\uparrow$, DP$\downarrow$) & (Test\_acc$\uparrow$, EOp$\downarrow$) & (Test\_acc$\uparrow$, EOd$\downarrow$) \\
            \midrule    
            $1\times$ &  (0.848, 0.084) & (0.876, 0.031) & (0.864, 0.030) \\
            $\sfrac{1}{2}\times$ & (0.872, 0.105) & (0.891, 0.042) & (0.880, 0.028)\\
            $\sfrac{1}{5}\times$ & (0.872, 0.117) & (0.863, 0.057) & (0.886, 0.028)\\
            \bottomrule
        \end{tabular}}
    \end{minipage}
    \captionof{figure}{\textbf{Left:} The impact of label budgets on the test accuracy \& DP gap in the \textbf{CelebA} dataset. \textbf{Right:} The impact of the validation set size on (\texttt{test\_accuracy}, \texttt{fairness\_violation}) results. }
    \label{table:validation_size_celeba}
\end{figure}

\section{Conclusions and limitations}\label{sec:conclusion}

In this work, we are interested in facilitating ML models that mitigate group fairness disparity without harming model accuracy. 
To achieve this, different from current active sampling methods, we propose a tractable fair influential sampling method FIS, which avoids the need for training group annotations during the sampling or training phase, thereby preventing the potential exposure of sensitive information.
In particular, this algorithm acquires data samples from a large dataset for training based on the influence of fairness and accuracy evaluated on the auxiliary validation dataset.
Empirical experiments on real-world data validate the efficacy of our proposed method.

Nonetheless, we recognize that our method has limitations. Although the proposed sampling algorithm does not require sensitive attribute information from the massive data, it relies on a clean and informative validation set that contains the sensitive attributes of data examples. We consider this as a reasonable requirement in practice, given the relatively modest size of the validation set.
Besides, one potential concern is that the sampling strategy may become inefficient when the collected validation set is noisy.
However, this practical issue can be heavily alleviated by using loss correction methods \citep{patrini2017making, zhu2023weak} or noise-tolerant fairness loss functions to rectify the error terms \citep{natarajan2013learning, wei2020combating, feng2021can}. 
In future work, we aim to address this limitation by developing more robust sampling strategies that can perform effectively even with noisy validation sets.

\paragraph{Acknowledgment}
This work is partially supported by the National Science Foundation (NSF) under grants IIS-2143895, IIS-2040800, IIS-2416896, and CCF-2023495. Additionally, Pang and Qian were partially supported by NSF Grants 2322919, 2420632, 2426031, and 2114113. 

\clearpage
\newpage

\section*{Broader Impact}
This paper presents work whose goal is to advance the field of fairness in machine learning. There are many potential societal consequences of our work. While the proposed algorithm does intend to infer sensitive attributes of data examples that may be protected by privacy regulations, it does not necessitate direct access to such sensitive information.  On the other hand, our work can serve as an effective approach leading to mitigating the disparity with a limited annotation budget. We have thoroughly examined the potential ethical implications of our work and, based on our assessment, do not identify any issues that we deem necessary to emphasize here specifically.

\bibliographystyle{plain}
\bibliography{ref}

\newpage
\appendix
\onecolumn

\section*{Appendix}\label{sec:appendix}

The Appendix is organized as follows.

\squishlist
    \item Section \ref{appendix:related_work} provides more details of the related work.
    \item Section \ref{appendix:connection_between_risk_and_EOd} explores the relationship between our proposed fairness notion risk disparity and common fairness metrics, such as DP and EOd. In particular, we provide the full proof for Proposition \ref{property:connections_to_DP_EOd}.
    \item Section \ref{appendix:fis_details} provides a detailed analysis of the FIS algorithm including 1) evaluating first-order influence estimations against real influence, 2) a comparative analysis of computational costs, and 3) the exploration of the labeling strategies. 
    \item Section \ref{appendix:omitted_proof} presents the full proofs for the Lemmas and Theorems shown in Section \ref{sec:method} and Section \ref{sec:theoretical_analysis}.
    \item Section \ref{appendix:experimental_settings} presents detailed descriptions of all datasets, corresponding parameter settings, and full version of the experimental results. In particular, to demonstrate FIS's advantage at the same levels of information, we introduce a new baseline called \texttt{Random+Val}.
    \item Section \ref{appendix:correlation_non-demongraphic_and_demographic} discusses the privacy concern potentially caused by the correlations between non-demographic and demographic features.
\squishend

\section{More details of related work}
\label{appendix:related_work}
\textbf{Active learning}
The core idea of active learning is to rank unlabeled instances by developing specific significant measures,
including uncertainty \citep{lewis1994heterogeneous, lindenbaum2004selective}, representativeness \citep{dasgupta2008hierarchical}, inconsistency \citep{wang2012inconsistency}, variance \citep{hoi2006large}, and error \citep{roy2001toward}. Each of these measures has its criterion to determine the importance of instances for enhancing classifier performance. For example, uncertainty considers the most important unlabeled instance to be the nearest one to the current classification boundary.
A related line of work \citep{liu2021influence, wang2022boosting, hammoudeh2022training, xie2023active} concentrates on ranking unlabeled instances based on the influence function. Compared to these studies with a focus on prediction accuracy, our work poses a distinct challenge taking into account fairness violations. We note that adopting a particular sampling strategy can lead to distribution shifts between the training and testing data. What's worse, even though fairness is satisfied within the training dataset, the model may still exhibit unfair treatments on the test dataset due to the distribution shift. Therefore, it becomes imperative for the sampling approach to also account for its potential impacts on fairness.

\textbf{Pareto optimality} In the field of fairness in machine learning, Pareto optimality indicates the theoretical frontier of fairness accuracy tradeoff, meaning that fairness can not be improved without worsening model accuracy. Existing methods primarily focus on seeking the Pareto frontier of fairness accuracy tradeoff for the neural network classifier, instead of reaching a better one.
For example, Balashankar et al. \citep{balashankar2019fair} first analyzes the Pareto optimality for classifiers within the context of fairness constraints.
Wang et al. \citep{wang2021understanding} explores the multi-dimensional Pareto frontiers of the fairness-accuracy tradeoff in the multi-task setting. Another works \citep{martinez2020minimax, martinez2019fairness} target to obtain a Pareto efficient classifier to reduce worst-case group risks by formulating group fairness as a multiple-objective optimization problem, where each group risk is an objective function.

\textbf{Fair classification}
The fairness-aware learning algorithms, in general, can be categorized into pre-processing, in-processing, and post-processing methods. Pre-processing methods typically reweigh or distort the data examples to mitigate the identified biases \citep{asudeh2019assessing,azzalini2021fair,tae2021slice, sharma2020data,celis2020data,chawla2002smote,zemel2013learning,chen2018my}.
More relevant to us is the \textit{importance reweighting}, which assigns different weights to different training examples to ensure fairness across groups \citep{kamiran2012data, jiang2020identifying, diesendruck2020importance,choi2020fair, qraitem2023bias, li2019repair}. %
Our proposed algorithm bears similarity to a specific case of importance reweighting, particularly the 0-1 reweighting applied to newly added data. The main advantage of our work, however, lies in its ability to operate without needing access to the sensitive attributes of either the new or training data. Other parallel studies utilize importance weighting to learn a complex fair generative model in a weakly supervised setting \citep{diesendruck2020importance, choi2020fair}, or to mitigate representation bias in training datasets \citep{li2019repair}. 
Post-processing methods typically enforce
fairness on a learned model through calibration \citep{feldman2015computational, feldman2015certifying, hardt2016equality}. Although this approach is likely to decrease the disparity of the classifier, by decoupling the training from the fairness enforcement, this procedure may not lead to the best trade-off between fairness and accuracy \citep{woodworth2017learning,pleiss2017fairness}. 
In contrast, our work can achieve a better trade-off between fairness and accuracy, because we reduce the fairness disparity by mitigating the adverse effects of distribution shifts on generalization error. 
Additionally, these post-processing techniques necessitate access to the sensitive attribute during the inference phase, which is often not available in many real-world scenarios.

\section{Understanding risk disparity through common fairness metrics}
\label{appendix:connection_between_risk_and_EOd}

Here, we provide a brief proof to illustrate Proposition \ref{property:connections_to_DP_EOd}, that is, the relationship between risk parity and common fairness metrics, such as DP and EOd. 
For completeness, here we provide a more detailed version of Proposition \ref{property:connections_to_DP_EOd}, as well as its proof.

\textbf{Proposition \ref{property:connections_to_DP_EOd}.} Consider a binary classification scenario involving two demographic groups $S\in \{k, k'\}$.
When two groups are balanced, i.e., $\frac{\mathbb{P}(Y=y|S=k)}{\mathbb{P}(Y=y|S=k')} =1$, where $\hat{Y}$ denotes the predicted label,  risk disparity can serve as a lower bound for fairness disparities based on EOd. Similarly, if the group sufficiency ratio can be calibrated by 1, i.e., $\mathbb{P}(Y=y|S=k, \hat{Y}=y) = \mathbb{P}(Y=y|S=k', \hat{Y}=y) = R$ where $R$ is calibration score, risk disparity can be formulated as a lower bound for DP-based fairness disparity.

\begin{proof}
Consider a scenario involving two demographic groups $S\in \{k, k'\}$, alongside a 0-1 loss function.

\paragraph{Fairness metric EOd} For EOd, the risk disparity can be reformulated as 
{\small
\begin{equation*}
\begin{aligned}
& |\mathcal{R}_{Q_k}(\mathbf{w}^P)-\mathcal{R}_{Q_{k^{\prime}}}(\mathbf{w}^P)| \\
= & \left| \frac{1}{|Q_k|} \sum\nolimits_{\left(x_n, y_n\right) \in Q_k} \mathbb{I}\left(\hat{y}_n \neq y_n\right)-\frac{1}{| Q_{k^{\prime}}|} \sum\nolimits_{\left(x_n, y_n\right) \in Q_{k^{\prime}}} \mathbb{I}\left(\hat{y}_n \neq y_n\right)\right| \\
= & \left|\mathbb{P}(\hat{Y} \neq Y \mid S=k)-\mathbb{P}(\hat{Y} \neq Y \mid S=k^{\prime})\right| \\
= &  \left|\sum_{y}\left(\mathbb{P}\left(Y=y \mid S=k^{\prime}\right) \mathbb{P}(\hat{Y}=y \mid S=k^{\prime}, Y=y)-\mathbb{P}(Y=y \mid S=k) \mathbb{P}(\hat{Y}=y \mid S=k, Y=y) \right) \right| \\
= &  \left|\sum_{y}\mathbb{P}(Y=y\mid S=k') \cdot \left(  \mathbb{P}(\hat{Y}=y \mid S=k^{\prime}, Y=y)- \omega_{\text{EOd}}(y) \mathbb{P}(\hat{Y}=y \mid S=k, Y=y)\right)\right| \\
=  & \bigg| \mathbb{P}(Y=1|S=k')\cdot \left(  \mathbb{P}(\hat{Y}=1 \mid S=k^{\prime}, Y=1)- \omega_{\text{EOd}}(y=1) \cdot \mathbb{P}(\hat{Y}=1 \mid S=k, Y=1) \right)  \quad \quad \text{\# binary case}\\ 
& \quad +\mathbb{P}(Y=0|S=k')\cdot \left(  \mathbb{P}(\hat{Y}=0 \mid S=k^{\prime}, Y=0)- \omega_{\text{EOd}}(y=0) \cdot \mathbb{P}(\hat{Y}=0 \mid S=k, Y=0) \right)  \bigg| \\
=  & \bigg| \mathbb{P}(Y=1|S=k')\cdot \left(  \mathbb{P}(\hat{Y}=1 \mid S=k^{\prime}, Y=1)- \omega_{\text{EOd}}(y=1) \cdot \mathbb{P}(\hat{Y}=1 \mid S=k, Y=1) \right)  \quad \quad \\ 
& \quad +\mathbb{P}(Y=0|S=k')\cdot \left( \omega_{\text{EOd}}(y=0) \cdot \mathbb{P}(\hat{Y}=1 \mid S=k, Y=0) - \mathbb{P}(\hat{Y}=1 \mid S=k^{\prime}, Y=0) \right) \\
& \quad + \mathbb{P}(Y=0|S=k') \cdot (1- \omega_{\text{EOd}}(y=0))\bigg| \\
=  & \bigg| \mathbb{P}(Y=1|S=k')\cdot \left(  \mathbb{P}(\hat{Y}=1 \mid S=k^{\prime}, Y=1)-  \mathbb{P}(\hat{Y}=1 \mid S=k, Y=1) \right)  \quad \quad \\ 
& \quad -\mathbb{P}(Y=0|S=k')\cdot \left( \mathbb{P}(\hat{Y}=1 \mid S=k^{\prime}, Y=0) - \mathbb{P}(\hat{Y}=1 \mid S=k, Y=0)\right) \bigg| \quad \quad \text{\# balanced groups}\\
\leq  & \mathbb{P}(Y=1|S=k')\cdot \bigg|  \mathbb{P}(\hat{Y}=1 \mid S=k^{\prime}, Y=1)-  \mathbb{P}(\hat{Y}=1 \mid S=k, Y=1) \bigg|  \\ 
& \quad + \mathbb{P}(Y=0|S=k')\cdot \bigg| \mathbb{P}(\hat{Y}=1 \mid S=k^{\prime}, Y=0) - \mathbb{P}(\hat{Y}=1 \mid S=k, Y=0)\bigg| \\
\leq  &\underbrace{  \sum_{y \in \{0,1\}} \bigg|\mathbb{P}(\hat{Y}=1 \mid S=k^{\prime}, Y=y)-  \mathbb{P}(\hat{Y}=1 \mid S=k, Y=y) \bigg| }_{\text{ EOd-based fairness disparity}}
\end{aligned}
\end{equation*}
}

where we define $\omega_{\text{EOd}}(y) := \frac{\mathbb{P}(Y=y|S=k)}{\mathbb{P}(Y=y|S=k')}$, serving as the bias weight. Here, we make a mild assumption that the two demographic groups are balanced, i.e., $\omega_{\text{EOd}}(y)=1$. Then, we can see that the last item measures the fairness disparity based on EOd for binary classification problems. Thus, we can claim that reducing risk disparity can promote the fairness metric EOd.

\paragraph{Fairness metric DP} Similarly, for DP, we formulate the risk disparity as 
{\small
\begin{equation*}
\begin{aligned}
& |\mathcal{R}_{Q_k}(\mathbf{w}^P)-\mathcal{R}_{Q_{k^{\prime}}}(\mathbf{w}^P)| \\
= & \left| \frac{1}{\left|Q_k\right|} \sum\nolimits_{\left(x_n, y_n\right) \in Q_k} \mathbb{I}\left(\hat{y}_n \neq y_n\right)-\frac{1}{| Q_{k^{\prime}}|} \sum\nolimits_{\left(x_n, y_n\right) \in Q_{k^{\prime}}} \mathbb{I}\left(\hat{y}_n \neq y_n\right)\right| \\
= & \left|\mathbb{P}(\hat{Y} \neq Y \mid S=k)-\mathbb{P}(\hat{Y} \neq Y \mid S=k^{\prime})\right| \\
= &  \left|\sum_{y}\left(\mathbb{P}(\hat{Y}=y \mid S=k^{\prime}) \mathbb{P}(Y=y \mid S=k^{\prime}, \hat{Y}=y)-\mathbb{P}(\hat{Y}=y \mid S=k) \mathbb{P}(Y=y \mid S=k, \hat{Y}=y) \right) \right| \\
= &  \left|\sum_{y} \mathbb{P}(Y=y|S=k', \hat{Y}=y) \cdot \left(  \mathbb{P}(\hat{Y}=y \mid S=k^{\prime})- \omega_{\text{DP}}(y) \mathbb{P}(\hat{Y}=y \mid S=k)\right)\right| \\
=  & \bigg|\mathbb{P}(Y=1|S=k', \hat{Y}=1) \cdot \left(  \mathbb{P}(\hat{Y}=1 \mid S=k^{\prime})- \omega_{\text{DP}}(y=1) \cdot \mathbb{P}(\hat{Y}=1 \mid S=k) \right)  \quad \quad \text{\# binary case}\\ 
& \quad +\mathbb{P}(Y=0|S=k', \hat{Y}=0) \cdot  \left(  \mathbb{P}(\hat{Y}=0 \mid S=k^{\prime})- \omega_{\text{DP}}(y=0)  \cdot \mathbb{P}(\hat{Y}=0 \mid S=k) \right)  \bigg| \\
=  & \bigg|\mathbb{P}(Y=1|S=k', \hat{Y}=1) \cdot \left(  \mathbb{P}(\hat{Y}=1 \mid S=k^{\prime})- \omega_{\text{DP}}(y=1) \cdot \mathbb{P}(\hat{Y}=1 \mid S=k) \right)  \quad \quad \\ 
& \quad +\mathbb{P}(Y=0|S=k', \hat{Y}=0) \cdot  \left( \omega_{\text{DP}}(y=0)  \cdot \mathbb{P}(\hat{Y}=1 \mid S=k) - \mathbb{P}(\hat{Y}=1 \mid S=k^{\prime}) \right) \\
& \quad + \mathbb{P}(Y=0|S=k', \hat{Y}=0) \cdot (1- \omega_{\text{DP}}(y=0))  \bigg| \quad \quad \quad \quad \text{\# Calibrating group sufficiency} \\
=  & \bigg|\mathbb{P}(Y=1|S=k', \hat{Y}=1) \cdot \left(  \mathbb{P}(\hat{Y}=1 \mid S=k^{\prime})- \mathbb{P}(\hat{Y}=1 \mid S=k) \right)  \quad \quad \\ 
& \quad +\mathbb{P}(Y=0|S=k', \hat{Y}=0) \cdot  \left( \mathbb{P}(\hat{Y}=1 \mid S=k) - \mathbb{P}(\hat{Y}=1 \mid S=k^{\prime}) \right)  \bigg| \\ 
=  & \bigg| \mathbb{P}(Y=1|S=k', \hat{Y}=1) - \mathbb{P}(Y=0|S=k', \hat{Y}=0)\bigg| \cdot \bigg | \mathbb{P}(\hat{Y}=1 \mid S=k^{\prime})- \mathbb{P}(\hat{Y}=1 \mid S=k) \bigg| \\
\leq   &  \underbrace{\bigg | \mathbb{P}(\hat{Y}=1 \mid S=k^{\prime})- \mathbb{P}(\hat{Y}=1 \mid S=k) \bigg| }_{\text{DP-based fairness disparity}}
\end{aligned}
\end{equation*}
}

where we define $\omega_{\text{DP}}(y) := \frac{\mathbb{P}(Y=y|S=k, \hat{Y}=y)}{\mathbb{P}(Y=y|S=k', \hat{Y}=y)}$. In fact, $\omega_{\text{DP}}(y)$  measures the group sufficiency ratio \citep{shui2022learning}.
Note that the group sufficiency is closely related to the idea of calibration \citep{barocas2023fairness}. Thus, we make a mild assumption that the group sufficiency ratio $\omega_{\text{DP}}(y) $ can be calibrated by 1, i.e., $\mathbb{P}(Y=y|S=k, \hat{Y}=y) = \mathbb{P}(Y=y|S=k', \hat{Y}=y) = R$, where $R$ is a calibration score. 
When two demographic groups are balanced, i.e., $\omega_{\text{DP}}(y)=1$, we can see that the last inequality indicates the DP-based fairness disparity for binary classification problems. Therefore, we can also claim that the fairness metric DP can also be encouraged when reducing the risk disparity.

\end{proof}

\section{Detailed analysis of the fair influential sampling algorithm}\label{appendix:fis_details}

\subsection{Evaluating first-order influence estimations against real influence}\label{appendix:influence_score}
Recall that the influence score derived in this paper primarily utilizes a first-order approach. Here, we will demonstrate how accurate the first-order estimation of the influence is in comparison to the real influence.
\begin{figure}
    \centering
    \includegraphics[width=0.5\linewidth]{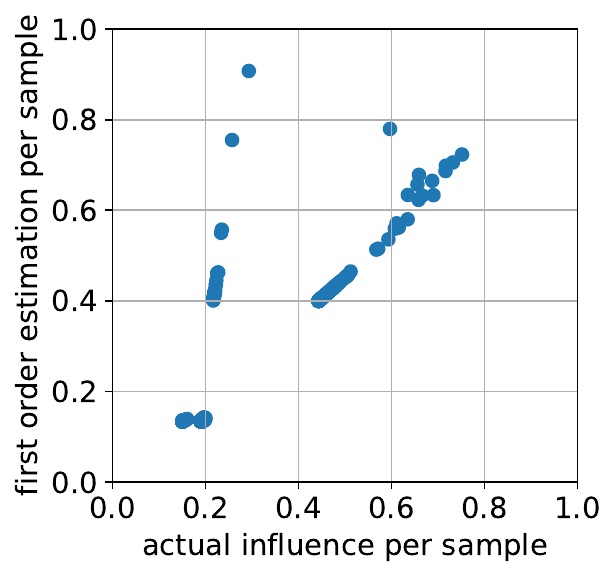}
    \caption{  We validate how accurate the first-order estimation of the influence is in comparison to the real influence. The $x$-axis represents the actual influence per sample, and the $y$-axis represents the estimated influence. We observe that while some of the examples are away from the diagonal line (which indicates the estimation is inaccurate), the estimated influences for most of the data samples are very close to their actual influence values.}
    \label{fig:effect-of-first-order}
\end{figure}

\subsection{Comparative analysis of computational costs} \label{apx:runtime_complexity}
Recall that the proposed algorithm FIS needs to pre-calculate the accuracy loss and fairness loss for evaluating the performance of a certain example.
However, the extra computation cost is comparable to the cost of traditional model training. Note that the main extra computation cost in FIS (Algorithm \ref{alg:fair_influential_sampling}) mainly results from model gradients.
Let $p$ denote the number of model parameters, then the cost for computing the gradients is $O(p)$ per sample. 
Specifically, in each round that involves sampling, we need to calculate three parts of gradients: the gradients of $|U|$ unlabeled instances, the average gradient of $|Q_v|$ validation instances w.r.t. accuracy loss, and the gradient of $|Q_v|$ validation instances w.r.t. fairness loss. Note that in general, $|Q_v| \ll |U|$.
In practical implementation, to speed up the calculation of gradients over $|U|$ instances, we randomly sample 0.2-0.5\% of the unlabeled dataset in each sampling batch. Additionally, we can increase the number of newly selected examples for each round ($r$) to save computation costs. In our experiments, we usually have 10-20 sampling rounds. For instance,  running one experiment for the CelebA dataset on a single GPU roughly requires about 4 hours.

\subsection{Exploration of the labeling strategies}
\label{appendix:labeling_strategy}

Note that we provide a strategy that employs lowest-influence labels for annotating labels.
This section will explore an alternative strategy that employs model predictions for the purpose of labeling. For completeness, we outline the two proposed labeling strategies as follows.

\begin{enumerate}[label=\textbf{Strategy \Roman*}, leftmargin=*, align=left, nosep]
    \item\label{method:strategy:min-infl} \textbf{Use low-influence labels.}
    That is, $\hat y =  \argmin_{k\in[K]} ~ |\infl{acc}(x', k) |$, which corresponds to using the most uncertain point.
    \item\label{method:strategy:max-pred} \textbf{Rely on model prediction.}
    That is, $\hat y =  \argmax_{k\in[K]} f(x;\mathbf{w})[k]$, where $f(x;\mathbf{w})[y]$ indicates the model's prediction probability on label $y$.
\end{enumerate}

\begin{remark}
    Suppose that the model is trained with cross-entropy loss. The labels obtained through \ref{method:strategy:max-pred} are sufficient to minimize the influence of the prediction component, i.e., $\infl{acc}(x', k)$. That said, the \ref{method:strategy:max-pred} will produce similar labels as \ref{method:strategy:min-infl}.
\end{remark}

\begin{proof}
    Based on the definition of the influence of the prediction component, as delineated in Eq. (\ref{eq:influence-of-prediction}), it becomes evident that the most uncertain points are obtained when the proxy labels closely align with the true labels. Consequently, the model predictions used in Strategy II also approximate the true labels to minimize the cross-entropy loss. Thus, in a certain sense, Strategy I and Strategy II can be considered equivalent.
\end{proof}

\section{Omitted proofs}
\label{appendix:omitted_proof}
In this section, we present complete proofs for the lemmas and theorems in Section \ref{sec:method} and \ref{sec:theoretical_analysis}, respectively.

\subsection{Proof of Lemma \ref{lemma:influence_of_predictions}}

\begin{varthm}[Lemma \ref{lemma:influence_of_predictions}]
    The influence of predictions on the validation dataset $Q_v$ can be denoted by
\begin{equation*}
\begin{aligned}
 \infl{acc}(z')
:=  \sum_{n \in |Q_v|} \infl{acc}(z', z_n^\val)   
\approx    - \eta  \sum\nolimits_{n \in |Q_v|} \left\langle \partial_{\mathbf{w}_{t}} \ell(\mathbf{w}_t, z'),  \partial_{\mathbf{w}_{t}} \ell(\mathbf{w}_t, z_n^\val)\right\rangle
\end{aligned}
\end{equation*}

\begin{proof}
    Taking the first-order Taylor expansion, we will have
\begin{align*}
         \ell(\mathbf{w}_{t+1}, z_n^\val)
\approx  \ell(\mathbf{w}_{t}, z_n^\val)
         + \left\langle \left. \frac{\partial  \ell(\mathbf{w}, z_n^\val)}{\partial f(\mathbf{w}, x_n^\val)} \right|_{\mathbf{w}=\mathbf{w}_{t}}, f(\mathbf{w}_{t+1}, x_n^\val) - f(\mathbf{w}_{t},x_n^\val) \right \rangle.
\end{align*}
where we take this expansion with respect to $f(\mathbf{w},x_n^\val)$.
Similarly, we have
\begin{align*}
         f(\mathbf{w}_{t+1}, x_n^\val) - f(\mathbf{w}_{t}, x_n^\val) 
\approx &  \left.\left\langle \frac{\partial f(\mathbf{w},x_n^\val)}{\partial \mathbf{w}}, \mathbf{w}_{t+1} -\mathbf{w}_t \right\rangle \right|_{\mathbf{w} = \mathbf{w}_{t}} \\
= &-\eta \left.\left\langle \frac{\partial f(\mathbf{w},x_n^\val)}{\partial \mathbf{w}}, \frac{\partial \ell(\mathbf{w},z')}{\partial \mathbf{w}} \right\rangle \right|_{\mathbf{w} = \mathbf{w}_{t}}.
\end{align*}
where the last equality holds due to Eq. (\ref{eq:change_model_parameters}).
Therefore, 
\begin{align*}
\ell(\mathbf{w}_{t+1}, z_n^\val) -  \ell(\mathbf{w}_{t}, z_n^\val)
\approx &  - \eta\left\langle \left. \frac{\partial  \ell(\mathbf{w}, z_n^\val)}{\partial f(\mathbf{w}, x_n^\val)} , \left\langle \frac{\partial f(\mathbf{w},x_n^\val)}{\partial \mathbf{w}}, \frac{\partial \ell(\mathbf{w},z')}{\partial \mathbf{w}} \right\rangle \right \rangle \right|_{\mathbf{w}=\mathbf{w}_{t}} \\
= & -\eta \left.\left\langle \frac{\partial f(\mathbf{w},x_n^\val)}{\partial \mathbf{w}}, \frac{\partial \ell(\mathbf{w},z')}{\partial \mathbf{w}} \right\rangle \right|_{\mathbf{w} = \mathbf{w}_{t}} .
\end{align*}
Then the accuracy influence on the validation dataset $V$ can be denoted by
\begin{align*}
 \infl{acc}(z')
:=  \sum_{n \in |Q_v|} \infl{acc}(z', z_n^\val)   
\approx    - \eta \left\langle \partial_{\mathbf{w}_{t}} \ell(\mathbf{w}_t, z'), \sum_{n \in |Q_v|} \partial_{\mathbf{w}_{t}} \ell(\mathbf{w}_t, z_n^\val)\right\rangle
\end{align*}
\end{proof}
\end{varthm}

\subsection{Proof of Lemma \ref{lemma:influence_of_fairness}}

\begin{varthm}[Lemma~\ref{lemma:influence_of_fairness}]
The influence of fairness on the validation dataset $Q_v$ can be denoted by
\begin{equation*}
    \begin{aligned}
         \infl{fair}(z') 
         : =   \sum_{n \in |Q_v|} \infl{fair}(z', z_n^\val) 
         \approx    - \eta \sum_{n\in|Q_v|} \left\langle \partial_{\mathbf{w}_t} \ell(\mathbf{w}_t, z'), \partial_{\mathbf{w}_t}  \phi(\mathbf{w}_t, z_n^\val) \right\rangle 
\end{aligned}
\end{equation*}
\begin{proof}

    By first-order approximation, we have
\begin{align*}
 \phi(\mathbf{w}_{t+1}, z_n^\val) 
\approx \phi(\mathbf{w}_{t}, z_n^\val) + \left\langle \left. \frac{\partial  \phi(\mathbf{w}_{t},z_n^\val) }{\partial f(\mathbf{w},x_n^\val)} \right|_{\mathbf{w}=\mathbf{w}_{t}}, f(\mathbf{w}_{t+1}, x_n^\val)  - f(\mathbf{w}_{t}, x_n^\val) \right \rangle.
\end{align*}

Recall by first-order approximation, we have
\begin{align*}
         f(\mathbf{w}_{t+1}, x_n^\val)  - f(\mathbf{w}_{t}, x_n^\val) 
\approx  -\eta \left.\left\langle \frac{\partial f(\mathbf{w},x_n^\val)}{\partial \mathbf{w}}, \frac{\partial \ell(\mathbf{w}, z')}{\partial \mathbf{w}} \right\rangle \right|_{\mathbf{w} = \mathbf{w}_{t}}.
\end{align*}

Therefore, 
\begin{align*}
 \phi(\mathbf{w}_{t+1}, z_n^\val) 
- \phi(\mathbf{w}_{t}, z_n^\val) \approx - \eta \left. \left\langle \frac{\partial \ell(\mathbf{w}, z')}{\partial \mathbf{w}}, \frac{\partial \phi(\mathbf{w}_{t}, z_n^\val)}{\partial \mathbf{w}} \right\rangle \right|_{\mathbf{w} = \mathbf{w}_{t}}
\end{align*}

Note the loss function in the above equation should be $\ell$ since the model is updated with $\ell$-loss.
Therefore, 
\begin{align*}
         \infl{fair}(z') = \sum_{n \in |Q_v|} \infl{fair}(z', z_n^\val)  \approx   - \left. \eta \sum_{n\in|Q_v|} \left\langle \frac{\partial \ell(\mathbf{w}, z')}{\partial \mathbf{w}}, \frac{\partial \phi(\mathbf{w}_{t}, z_n^\val)}{\partial \mathbf{w}} \right\rangle \right|_{\mathbf{w} = \mathbf{w}_{t}}.
\end{align*}
\end{proof}
\end{varthm}

\subsection{Proof of Theorem~\ref{lemma:accuracy_bound}}

\begin{varthm}[Theorem~\ref{lemma:accuracy_bound}]
     (Generalization error bound). Let $\text{dist}(\mathcal{P}, \mathcal{Q})$, $G_P$ be defined therein. With probability at least $1-\delta$ with $\delta \in (0,1)$, the generalization error bound of the model trained on dataset $P$ is
         \begin{align}
        \mathcal{R}_{\mathcal{Q}}(\mathbf{w}^P) \leq   \underbrace{G_P \cdot \text{dist}(\mathcal{P}, \mathcal{Q})}_{\text{ distribution shift}} + \sqrt{\frac{\log(4/\delta)}{2|P|}} +  \mathcal{R}_{P}(\mathbf{w}^{P}).
    \end{align}

Note that the generalization error bound is predominantly influenced by the shift in distribution when we think of an overfitting model, i.e., the empirical risk $\mathcal{R}_{P}(\mathbf{w}^{P}) \to 0$. The detailed proof is presented as follows.

\begin{proof}
The generalization error bound is
\begin{equation*}
    \begin{split}
        \mathcal{R}_{\mathcal{Q}}(\mathbf{w}^P) & = \underbrace{\bigg ( \mathcal{R}_{\mathcal{Q}}(\mathbf{w}^P) - \mathcal{R}_{\mathcal{P}}(\mathbf{w}^P) \bigg ) }_{\text{distribution shift}} + \underbrace{\bigg (\mathcal{R}_{\mathcal{P}}(\mathbf{w}^P) - \mathcal{R}_{P}(\mathbf{w}^P)\bigg)}_{\text{Hoeffding's inequality}} + \underbrace{\mathcal{R}_{P}(\mathbf{w}^P)) }_{\text{empirical risk}} \\
        & \leq G_P\cdot \text{dist}(\mathcal{P}, \mathcal{Q}) + \sqrt{\frac{\log(4/\delta)}{2|P|}} + \mathcal{R}_{P}(\mathbf{w}^P)  \\
    \end{split}
\end{equation*}
For the first term (distribution shift), we have
\begin{equation*}
    \begin{split}
         \mathcal{R}_{\mathcal{Q}}(\mathbf{w}^P) - \mathcal{R}_{\mathcal{P}}(\mathbf{w}^P) 
        &= \mathbb{E}_{z\sim \mathcal{Q}} [\ell( \mathbf{w}^P, z)] - \mathbb{E}_{z\sim \mathcal{P}}[\ell( \mathbf{w}^P, z)] \\
        & = \sum_{i=1}^{I}p^{(\mathcal{Q})}(\pi=i)\mathbb{E}_{z\sim \pi_i} [\ell( \mathbf{w}^P, z)] - \sum_{i=1}^{I}p^{(\mathcal{P})}(\pi=i)\mathbb{E}_{z\sim \pi_i}  [\ell( \mathbf{w}^P, z)] \\
        & \leq  \sum_{i=1}^{I} | p^{(\mathcal{P})}(\pi=i) - p^{(\mathcal{Q})}(\pi=i) |  \mathbb{E}_{z\sim \pi_i}  [\ell( \mathbf{w}^P, z)] \\
        & \leq G_P\cdot \text{dist}(\mathcal{P}, \mathcal{Q}).
    \end{split}
\end{equation*}

where we define $\text{dist}(\mathcal{P}, \mathcal{Q}) =  \sum_{i=1}^{I} | p^{(\mathcal{P})}(\pi=i) - p^{(\mathcal{Q})}(\pi=i) | $ and $\mathbb{E}_{z\sim \pi_i}  [\ell( \mathbf{w}^P, z)] \leq G_P, \forall i \in I$ because of Assumption \ref{assumption:bounded_gradient}. To avoid misunderstanding, we use a subscript $P$ of the constant $G$ to clarify the corresponding model $\mathbf{w}^P$.
Then, for the second term (Hoeffding inequality), with probability at least $1-\delta$, we have $|\mathcal{R}_{\mathcal{P}}(\mathbf{w}^P) - \mathcal{R}_{P}(\mathbf{w}^{P})| \leq \sqrt{\frac{\log(4/\delta)}{2|P|}}$. 
\end{proof}
\end{varthm}

\subsection{Proof of Theorem \ref{thm:performance_gap_test}}

\begin{varthm}[Theorem \ref{thm:performance_gap_test}]
(Upper bound of fairness disparity). Suppose  $\mathcal{R}_{\mathcal{Q}}(\cdot)$ follows Assumption \ref{assumption:lipschitz}. Let $\text{dist}(\mathcal{P}, \mathcal{Q})$, $G_P$, $\text{dist}(\mathcal{P}_k, \mathcal{Q}_k)$ and $\text{dist}(P_k, P)$ be defined therein. The initial learning rate $\eta_0^2 < \frac{1}{\sqrt{2}TL}$, where $T$ denotes the number of training epochs. Given model $\mathbf{w}^P$ and $\mathbf{w}^k$ trained exclusively on group $k$'s data $P_k$,  with probability at least $1-\delta$ with $\delta \in (0,1)$, then the upper bound of the fairness disparity is
\begin{equation*}
\begin{split}
         \mathcal{R}_{\mathcal{Q}_k}(\mathbf{w}^{P})- \mathcal{R}_{\mathcal{Q}}(\mathbf{w}^P) \leq 
            \underbrace{G_P \cdot \text{dist}(\mathcal{P}, \mathcal{Q})}_{\text{ distribution shift}} + \underbrace{4 L^2  G^2 \cdot  \text{dist}(P_k, P)^2}_{\text{group gap}}  + G_k \cdot \text{dist}(\mathcal{P}_k, \mathcal{Q}_k) + \Upsilon.
\end{split}
\end{equation*}
where $\Upsilon =\sqrt{\frac{\log(4/\delta)}{2|P|}} +  \sqrt{\frac{\log(4/\delta)}{2|P_k|}} + \varpi  + \varpi_k$.
Note that $\mathbb{E}_{z\sim \pi_i} [\ell(\mathbf{w}^k, z)] \leq G_k$,  $\varpi = \mathcal{R}_{P}(\mathbf{w}^P) - \mathcal{R}_{\mathcal{Q}}^*(\mathbf{w}^{\mathcal{Q}})$ and $\varpi_k = \mathcal{R}_{P_k}(\mathbf{w}^{k})  - \mathcal{R}_{\mathcal{Q}_k}^*(\mathbf{w}^{\mathcal{Q}_k}) $. Specifically, $\varpi$ and $\varpi_k$ can be regarded as constants because  $\mathcal{R}_{P}(\mathbf{w}^P)$ and $\mathcal{R}_{P_k}(\mathbf{w}^{k}) $ correspond to the empirical risks, $\mathcal{R}_{\mathcal{Q}}^*(\mathbf{w^{\mathcal{Q}}})$ and $\mathcal{R}_{\mathcal{Q}_k}^*(\mathbf{w^{\mathcal{Q}_k}})$ represent the ideal minimal empirical risk of model $\mathbf{w}^{\mathcal{Q}}$ trained on distribution $\mathcal{Q}$ and $\mathcal{Q}_k$, respectively. Moreover, these quantities $\varpi$ and $\varpi_k$ are not taken into account during the training phase, but rather in relation to the final model.

\begin{proof}
First  of all, we have
\begin{align*}
    &\mathcal{R}_{\mathcal{Q}_k}(\mathbf{w}^{P})- \mathcal{R}_{\mathcal{Q}}(\mathbf{w}^P) \\
     =&    (\mathcal{R}_{\mathcal{Q}}(\mathbf{w}^{P_k})-  \mathcal{R}_{\mathcal{Q}}(\mathbf{w}^P)  ) +  (\mathcal{R}_{\mathcal{Q}_k}(\mathbf{w}^{P})-  \mathcal{R}_{\mathcal{Q}}(\mathbf{w}^{P_k})  ) \\
    = &    (\mathcal{R}_{\mathcal{Q}}(\mathbf{w}^{P_k})-  \mathcal{R}_{\mathcal{Q}}(\mathbf{w}^P)  ) +  (\mathcal{R}_{\mathcal{Q}_k}(\mathbf{w}^{P})-  \mathcal{R}_{\mathcal{Q}_k}(\mathbf{w}^{P_k})  ) +  (\mathcal{R}_{\mathcal{Q}_k}(\mathbf{w}^{P_k})-  \mathcal{R}_{\mathcal{Q}}(\mathbf{w}^{P_k})  ) \\
   \leq  & (\mathcal{R}_{\mathcal{Q}}(\mathbf{w}^{P_k})-  \mathcal{R}_{\mathcal{Q}}(\mathbf{w}^P)  ) +  (\mathcal{R}_{\mathcal{Q}_k}(\mathbf{w}^{P})-  \mathcal{R}_{\mathcal{Q}_k}(\mathbf{w}^{P_k})  )
\end{align*}
where $\mathbf{w}^{P_k}$  represents the model trained exclusively on group $k$'s data. For simplicity, when there is no confusion, we use $\mathbf{w}^{k}$ to substitute $\mathbf{w}^{P_k}$. The inequality \( \mathcal{R}_{\mathcal{Q}_k}(\mathbf{w}^{k}) - \mathcal{R}_{\mathcal{Q}}(\mathbf{w}^{k}) \geq 0 \) holds because the model tailored for a single group $k$ can not generalize well to the entirety of the test set $Q$.

Then, for the first term, we have
\begin{equation*}
\begin{split}
    \mathcal{R}_{\mathcal{Q}}(\mathbf{w}^k)-  \mathcal{R}_{\mathcal{Q}}(\mathbf{w}^P) & \stackrel{(a)}{\leq} \langle \nabla  \mathcal{R}_{\mathcal{Q}}(\mathbf{w}^P), \mathbf{w}^{k} - \mathbf{w}^P \rangle  + \frac{L}{2} \lVert \mathbf{w}^{k} - \mathbf{w}^P \rVert^2  \\
    & \stackrel{(b)}{\leq} L \lVert \mathbf{w}^{k} - \mathbf{w}^P \rVert^2  + \frac{1}{2L} \lVert  \nabla \mathcal{R}_{\mathcal{Q}}(\mathbf{w}^P) \rVert^2 \\
    & \stackrel{(c)}{\leq} \underbrace{L \lVert \mathbf{w}^{k} - \mathbf{w}^P \rVert^2 }_{\text{group gap}}  +  \underbrace{(\mathcal{R}_{\mathcal{Q}}(\mathbf{w}^P) - \mathcal{R}_{\mathcal{Q}}^*(\mathbf{w^{\mathcal{Q}}}))}_{\text{train-test model gap}}
\end{split}
\end{equation*}
where inequality (a) holds because of the L-smoothness of expected loss $\mathcal{R}_{\mathcal{Q}}(\cdot)$, i.e., Assumption \ref{assumption:lipschitz}. Specifically, inequality (b) holds because, by Cauchy-Schwarz inequality and AM-GM inequality, we have
\begin{equation*}
    \langle \nabla  \mathcal{R}_{\mathcal{Q}}(\mathbf{w}^P), \mathbf{w}^{k} - \mathbf{w}^P \rangle \leq  \frac{L}{2} \lVert \mathbf{w}^{k} - \mathbf{w}^P \rVert^2 + \frac{1}{2L} \lVert  \nabla \mathcal{R}_{\mathcal{Q}}(\mathbf{w}^P) \rVert^2.
\end{equation*}

Then, inequality (c) holds due to the L-smoothness of $\mathcal{R}_{\mathcal{Q}}(\cdot)$ (Assumption \ref{assumption:lipschitz}), we can get a variant of Polak-Łojasiewicz inequality, which follows 
\begin{equation*}
    \lVert   \nabla \mathcal{R}_{\mathcal{Q}}(\mathbf{w}^P) \rVert^2 \leq 2L(\mathcal{R}_{\mathcal{Q}}(\mathbf{w}^P) - \mathcal{R}_{\mathcal{Q}}^*(\mathbf{w^{\mathcal{Q}}})).
\end{equation*}
where $(\mathcal{R}_{\mathcal{Q}}^*(\mathbf{w^{\mathcal{Q}}}))$ denotes the ideal minimal empirical risk of model $\mathbf{w}^{\mathcal{Q}}$ trained on distribution $\mathcal{Q}$.
Following a similar idea, for the second term, we also have 
\begin{equation*}
\begin{split}
    \mathcal{R}_{\mathcal{Q}_k}(\mathbf{w}^P)-  \mathcal{R}_{\mathcal{Q}_k}(\mathbf{w}^{k}) 
    & \leq L \lVert \mathbf{w}^{P} - \mathbf{w}^k \rVert^2   +  (\mathcal{R}_{\mathcal{Q}_k}(\mathbf{w}^k) - \mathcal{R}_{\mathcal{Q}_k}^*(\mathbf{w}^{\mathcal{Q}_k}))
\end{split}
\end{equation*}

Combined with two terms, we have
\begin{align*}
    \mathcal{R}_{\mathcal{Q}_k}(\mathbf{w}^{P})- \mathcal{R}_{\mathcal{Q}}(\mathbf{w}^P) & \leq  \underbrace{(\mathcal{R}_{\mathcal{Q}}(\mathbf{w}^P) - \mathcal{R}_{\mathcal{Q}}^*(\mathbf{w^{\mathcal{Q}}}))}_{\text{train-test model gap}} + \underbrace{2L \lVert \mathbf{w}^{k} - \mathbf{w}^P \rVert^2 }_{\text{group model gap}}   + (\mathcal{R}_{\mathcal{Q}_k}(\mathbf{w}^k) - \mathcal{R}_{\mathcal{Q}_k}^*(\mathbf{w}^{\mathcal{Q}_k}))
\end{align*}

Lastly, integrating with Lemmas \ref{lemma:train-test_model_gap}, \ref{lemma:diff_w_k-w} and \ref{lemma:train-test_model_gap_k}, we can finish the proof.
\end{proof}
\end{varthm}

\begin{lemma}
\label{lemma:train-test_model_gap}
(Train-test model gap) With probability at least $1-\delta$, given the  model $\mathbf{w}^P$ trained on train set $P$, we have
\begin{align*}
    \mathcal{R}_{\mathcal{Q}}(\mathbf{w}^P) - \mathcal{R}_{\mathcal{Q}}^*(\mathbf{w^{\mathcal{Q}}}) \leq G_P \cdot \text{dist}(\mathcal{P}, \mathcal{Q}) + \sqrt{\frac{\log(4/\delta)}{2|P|}} + \varpi.
\end{align*}
where $\text{dist}(\mathcal{P}, \mathcal{Q}) =  \sum_{i=1}^{I} | p^{(\mathcal{P})}(\pi=i) - p^{(\mathcal{Q})}(\pi=i) | $ and $\mathbb{E}_{z\sim \pi_i} [\ell( \mathbf{w}^P, z)] \leq G_P, \forall i \in I$, and a constant $\varpi := \mathcal{R}_{P}(\mathbf{w}^P) - \mathcal{R}_{\mathcal{Q}}^*(\mathbf{w^{\mathcal{Q}}})$.
\end{lemma}

\begin{proof}
First of all, we have,
\begin{equation*}
    \begin{split}
        \mathcal{R}_{\mathcal{Q}}(\mathbf{w}^P) - \mathcal{R}_{\mathcal{Q}}^*(\mathbf{w^{\mathcal{Q}}}) & = \bigg ( \mathcal{R}_{\mathcal{Q}}(\mathbf{w}^P) - \mathcal{R}_{\mathcal{P}}(\mathbf{w}^P) \bigg ) + \mathcal{R}_{\mathcal{P}}(\mathbf{w}^P) - \mathcal{R}_{\mathcal{Q}}^*(\mathbf{w^{\mathcal{Q}}}) \\
        & \leq G\cdot \text{dist}(\mathcal{P}, \mathcal{Q})  + \mathcal{R}_{\mathcal{P}}(\mathbf{w}^P) - \mathcal{R}_{\mathcal{Q}}^*(\mathbf{w^{\mathcal{Q}}}) \\
        & \leq \underbrace{G\cdot \text{dist}(\mathcal{P}, \mathcal{Q}) }_{\text{distribution shift}} + \underbrace{\bigg (\mathcal{R}_{\mathcal{P}}(\mathbf{w}^P) - \mathcal{R}_{P}(\mathbf{w}^P)\bigg)}_{\text{Hoeffding's inequality}} + \underbrace{\bigg(\mathcal{R}_{P}(\mathbf{w}^P) - \mathcal{R}_{\mathcal{Q}}^*(\mathbf{w^{\mathcal{Q}}}) \bigg)}_{\text{overfitting \& ideal case}}\\
    \end{split}
\end{equation*}

For the first term (distribution shift), we have
\begin{equation*}
    \begin{split}
         \mathcal{R}_{\mathcal{Q}}(\mathbf{w}^P) - \mathcal{R}_{\mathcal{P}}(\mathbf{w}^P) 
        & = \mathbb{E}_{z\sim \mathcal{Q}} [\ell( \mathbf{w}^P, z)] - \mathbb{E}_{z\sim \mathcal{P}}[\ell( \mathbf{w}^P, z)] \\
        & = \sum_{i=1}^{I}p^{(\mathcal{Q})}(\pi=i)\mathbb{E}_{z\sim \pi_i} [\ell( \mathbf{w}^P, z)] - \sum_{i=1}^{I}p^{(\mathcal{P})}(\pi=i)\mathbb{E}_{z\sim \pi_i}  [\ell( \mathbf{w}^P, z)] \\
        & \leq  \sum_{i=1}^{I} | p^{(\mathcal{P})}(\pi=i) - p^{(\mathcal{Q})}(\pi=i) |  \mathbb{E}_{z\sim \pi_i}  [\ell( \mathbf{w}^P, z)] \\
        & \leq G_P\cdot \text{dist}(\mathcal{P}, \mathcal{Q}).
    \end{split}
\end{equation*}

where we define $\text{dist}(\mathcal{P}, \mathcal{Q}) =  \sum_{i=1}^{I} | p^{(\mathcal{P})}(\pi=i) - p^{(\mathcal{Q})}(\pi=i) | $ and $\mathbb{E}_{z\sim \pi_i}  [\ell( \mathbf{w}^P, z)] \leq G_P, \forall i \in I$ because of Assumption \ref{assumption:bounded_gradient}.
For the second term, with probability at least $1-\delta$, we have $|\mathcal{R}_{\mathcal{P}}(\mathbf{w}^P) - \mathcal{R}_{P}(\mathbf{w}^P)| \leq \sqrt{\frac{\log(4/\delta)}{2|P|}}$.
Note that the third term $\mathcal{R}_{P}(\mathbf{w}^P) - \mathcal{R}_{Q}^*(\mathbf{w^{\mathcal{Q}}})$ can be regarded as a constant $\varpi$.because $\mathcal{R}_{P}(\mathbf{w}^P)$ is the empirical risk and $\mathcal{R}_{\mathcal{Q}}^*(\mathbf{w^{\mathcal{Q}}})$ is the ideal minimal empirical risk of model $\mathbf{w}^{\mathcal{Q}}$ trained on distribution $\mathcal{Q}$.

Therefore, with probability at least $1-\delta$, given model $\mathbf{w}^P$, 
\begin{align*}
    \mathcal{R}_{\mathcal{Q}}(\mathbf{w}^P) - \mathcal{R}_{\mathcal{Q}}^*(\mathbf{w^{\mathcal{Q}}}) \leq G_P \cdot \text{dist}(\mathcal{P}, \mathcal{Q}) + \sqrt{\frac{\log(4/\delta)}{2|P|}} + \varpi.
\end{align*}
\end{proof}

\begin{lemma}
\label{lemma:diff_w_k-w}
    (Group model gap) Suppose Assumptions \ref{assumption:lipschitz} and \ref{assumption:bounded_gradient} hold for empirical risk $\mathcal{R}_{P}(\cdot)$. The initial learning rate $\eta_0^2 < \frac{1}{\sqrt{2}TL}$, where $T$ denotes the number of training epochs. Then, we have
\begin{equation*}
    \lVert \mathbf{w}^{k} - \mathbf{w}^{P} \rVert^2 \leq 2 L  G^2 \bigg( \sum_{i=1}^{I}\bigg|p^{(k)}(\pi=i) - p^{(P)}(\pi=i)\bigg|\bigg)^2.
\end{equation*}
\end{lemma}

\begin{proof}
According to the above definition, we similarly define the following empirical risk $\mathcal{R}_{P_k}(\mathbf{w})$ 
 over group $k$'s data $P_k$ by splitting samples according to their marginal distributions, shown as follows.

\begin{equation*}
    \mathcal{R}_{P_k}(\mathbf{w}) := \sum_{i=1}^{I}p^{(k)}(\pi=i)\mathbb{E}_{z\sim \pi_i} [\ell(\mathbf{w}, z)].
\end{equation*}

Let $\eta_t$ indicate the learning rate of epoch $t$. Then, for each epoch $t$, group $k$'s optimizer performs SGD as follows:
\begin{equation*}
   \mathbf{w}_t^{k} = \mathbf{w}_{t-1}^{k} - \eta_t \sum_{i=1}^{I}p^{(k)}(\pi=i) \nabla_{\mathbf{w}}\mathbb{E}_{z\sim \pi_i} [\ell(\mathbf{w}_{t-1}^{k}, z)].
\end{equation*}

For any epoch $t+1$, we have
\begin{equation*}
    \begin{split}
        & \lVert \mathbf{w}^{k}_{t+1} - \mathbf{w}^{P}_{t+1} \rVert^2 \\
        & = \lVert \mathbf{w}^{k}_{t} - \eta_t \sum_{i=1}^{I}p^{(k)}(\pi =i) \nabla_{\mathbf{w}} \mathbb{E}_{z\sim \pi_i} [\ell( \mathbf{w}_{t}^{k}, z)] - \mathbf{w}^{P}_{t} + \eta_t \sum_{i=1}^{I}p^{(P)}(\pi =i) \nabla_{\mathbf{w}} \mathbb{E}_{z\sim \pi_i} [\ell(\mathbf{w}_{t}^{P},z)]\rVert^2 \\ 
        & \leq \lVert \mathbf{w}^{k}_{t} - \mathbf{w}^{P}_{t} \rVert^2 + \eta_t^2 \lVert \sum_{i=1}^{I}p^{(k)}(\pi =i) \nabla_{\mathbf{w}} \mathbb{E}_{z\sim \pi_i} [\ell( \mathbf{w}_{t}^{k}, z)]- \sum_{i=1}^{I}p^{(P)}(\pi =i) \nabla_{\mathbf{w}} \mathbb{E}_{z\sim \pi_i} [\ell(\mathbf{w}_{t}^{P},z)] \rVert^2 \\
        & \leq \lVert \mathbf{w}^{k}_{t} - \mathbf{w}^{P}_{t} \rVert^2 + 2\eta_t^2 \lVert \sum_{i=1}^{I}p^{(P)}(\pi =i)L_{\pi_i} \bigg[\nabla_{\mathbf{w}} \mathbb{E}_{z\sim \pi_i} [\ell( \mathbf{w}_{t}^{k}, z)]- \nabla_{\mathbf{w}} \mathbb{E}_{z\sim \pi_i} [\ell(\mathbf{w}_{t}^{P},z)]\bigg ]\rVert^2 \\
        & \quad + 2\eta_t^2  \lVert \sum_{i=1}^{I}\bigg(p^{(k)}(\pi =i) - p^{(P)}(\pi =i)\bigg)\nabla_{\mathbf{w}} \mathbb{E}_{z\sim \pi_i} [\ell(\mathbf{w}_{t}^{P},z)]\rVert^2 \\
        & \leq \lVert \mathbf{w}^{k}_{t} - \mathbf{w}^{P}_{t} \rVert^2 + 2\eta_t^2 \bigg(\sum_{i=1}^{I}p^{(k)}(\pi =i)L_{\pi_i}\bigg)^2\lVert \mathbf{w}^{k}_{t} - \mathbf{w}^{P}_{t} \rVert^2  \\
        & \quad + 2 L \eta_t^2 g^2_{max}(\mathbf{w}^{Q}_{t}) \bigg( \sum_{i=1}^{I}|p^{(k)}(\pi =i) - p^{(P)}(\pi =i)|\bigg)^2 \\
        & \leq \bigg(1 + 2\eta_t^2 \bigg(\sum_{i=1}^{I}p^{(k)}(\pi =i)L_{\pi_i}\bigg)^2\bigg) \lVert \mathbf{w}^{k}_{t} - \mathbf{w}^{P}_{t} \rVert^2 \\
         & \quad + 2 L \eta_t^2 g^2_{max}(\mathbf{w}^{Q}_{t}) \bigg( \sum_{i=1}^{I}|p^{(k)}(\pi =i) - p^{(P)}(\pi =i)|\bigg)^2 \\
        & \leq (1 + 2\eta_t^2 L^2) \lVert \mathbf{w}^{k}_{t} - \mathbf{w}^{P}_{t} \rVert^2  +  2 L \eta_t^2G^2 \bigg( \sum_{i=1}^{I}|p^{(k)}(\pi =i) - p^{(P)}(\pi =i)|\bigg)^2.
    \end{split}
\end{equation*}
where the third inequality holds since we assume that $\nabla_{\mathbf{w}} \mathbb{E}_{z\sim \pi_i} [\ell(\mathbf{w}, z)]$  is $L_{\pi_i}$-Lipschitz continuous, {\em i.e.}, $|| \nabla_{\mathbf{w}} \mathbb{E}_{z\sim \pi_i} [\ell( \mathbf{w}_{t}^{k}, z)]-\nabla_{\mathbf{w}} \mathbb{E}_{z\sim \pi_i} [\ell(\mathbf{w}_{t}^{P},z)]|| \leq L_{\pi_i} ||\mathbf{w}_{t}^{k}-\mathbf{w}_{t}^{P}||$, and denote $g_{max}(\mathbf{w}_{t}^{P}) = \max_{i=1}^{I} ||\nabla_{\mathbf{w}} \mathbb{E}_{z\sim \pi_i} [\ell(\mathbf{w}_{t}^{P},z)]|| $. The last inequality holds because the above-mentioned assumption that $L = L_{\pi_i} = L_{\pi}, \forall i \in I$, {\em i.e.}, Lipschitz-continuity will not be affected by the samples' classes.
Then, $g_{max}(\mathbf{w}_{t}^{P}) \leq G$ because of Assumption \ref{assumption:bounded_gradient}. 

For $T$ training epochs, we have
\begin{equation*}
    \begin{split}
        & \lVert \mathbf{w}^{k}_{T} - \mathbf{w}^{P}_{T} \rVert^2 \\ 
        & \leq (1 + 2\eta_T^2 L^2) \lVert \mathbf{w}^{k}_{T-1} - \mathbf{w}^{P}_{T-1} \rVert^2  +  2 L \eta_T^2 G^2 \bigg( \sum_{i=1}^{I}|p^{(k)}(\pi =i) - p^{(P)}(\pi =i)|\bigg)^2\\
        & \leq \prod_{t=0}^{T} (1+2\eta_{t}^2 L^2)^{t}  \lVert \mathbf{w}^{k}_{0} - \mathbf{w}^{P}_{0} \rVert^2  +  2 L  G^2 \sum_{t=0}^{T}(\eta_{t}^{2}(1+2\eta_{t}^2 L^2))^{T-t}  \bigg( \sum_{i=1}^{I}|p^{(k)}(\pi =i) - p^{(P)}(\pi =i)|\bigg)^2 \\
        & \leq 2 L  G^2 \sum_{t=0}^{T}(\eta_{t}^{2}(1+2\eta_{t}^2 L^2))^{T-t}  \bigg( \sum_{i=1}^{I}|p^{(k)}(\pi =i) - p^{(P)}(\pi =i)|\bigg)^2.
    \end{split}
\end{equation*}
where the last inequality holds because the initial models are the same, i.e., $ \mathbf{w}_{0} = \mathbf{w}^{k}_{0} = \mathbf{w}^{P}_{0}, \forall k$.  When the condition $\eta_t^2 < \eta_0^2 < \frac{1}{\sqrt{2}TL}$ satisfies,  $2 L  G^2 \sum_{t=0}^{T}(\eta_{t}^{2}(1+2\eta_{t}^2 L^2))^{T-t}$ can be simplifited as $2 L G^2$, which is independent of the learning algorithm. This condition is easy to be satisfied since the learning rate $\eta_t$ is a small value ($<0.0001$) and usually set to be decay with the training epoch (i.e., $\eta_{t+1} \leq \eta_t$). 
\end{proof}

\begin{lemma}
\label{lemma:train-test_model_gap_k}
 With probability at least $1-\delta$, given the  model $\mathbf{w}^{k}$ trained on group $k$'s dataset $P_k$, we have
\begin{align*}
    \mathcal{R}_{\mathcal{Q}_k}(\mathbf{w}^k) - \mathcal{R}_{\mathcal{Q}_k}^*(\mathbf{w}^{\mathcal{Q}_k}) \leq G_k \cdot \text{dist}(\mathcal{P}_k, \mathcal{Q}_k) + \sqrt{\frac{\log(4/\delta)}{2|P_k|}} + \varpi_k.
\end{align*}
where $\text{dist}(\mathcal{P}_k, \mathcal{Q}_k) =  \sum_{i=1}^{I} | p^{(\mathcal{P}_k)}(\pi=i) - p^{(\mathcal{Q}_k)}(\pi=i) | $ and $\mathbb{E}_{z\sim \pi_i} [\ell(\mathbf{w}^k, z)] \leq G_k, \forall i \in I$, and $\varpi_k := \mathcal{R}_{P_k}(\mathbf{w}^{k})  - \mathcal{R}_{\mathcal{Q}_k}^*(\mathbf{w}^{\mathcal{Q}_k}) $.
\end{lemma}

\begin{proof}
Building upon the proof idea presented in Lemma~\ref{lemma:train-test_model_gap}, for completeness, we provide a full proof here.
Firstly, we have,
\begin{equation*}
    \begin{split}
         & \mathcal{R}_{\mathcal{Q}_k}(\mathbf{w}^k) - \mathcal{R}_{\mathcal{Q}_k}^*(\mathbf{w}^{\mathcal{Q}_k})  \\
         & = \underbrace{( \mathcal{R}_{\mathcal{Q}_k}(\mathbf{w}^k) - \mathcal{R}_{\mathcal{P}_k}(\mathbf{w}^{k}) )}_{\text{distribution shift}}  + \underbrace{  (\mathcal{R}_{\mathcal{P}_k}(\mathbf{w}^{k}) - \mathcal{R}_{P_k}(\mathbf{w}^{k}) )}_{\text{Hoeffding's inequality}} + \underbrace{ (\mathcal{R}_{P_k}(\mathbf{w}^{k})  - \mathcal{R}_{\mathcal{Q}_k}^*(\mathbf{w}^{\mathcal{Q}_k})  )}_{\text{overfitting \& ideal case}}
    \end{split}
\end{equation*}

For the first term, we have
\begin{equation*}
    \begin{split}
        & \mathcal{R}_{\mathcal{Q}_k}(\mathbf{w}^k) - \mathcal{R}_{\mathcal{P}_k}(\mathbf{w}^{k}) \\
        & = \sum_{i=1}^{I}p^{(\mathcal{Q}_k)}(\pi=i)\mathbb{E}_{z\sim \pi_i} [\ell(\mathbf{w}^k, z)] - \sum_{i=1}^{I}p^{(\mathcal{P}_k)}(\pi=i)\mathbb{E}_{z\sim \pi_i}  [\ell(\mathbf{w}^k, z)] \\
        & \leq  \sum_{i=1}^{I} | p^{(\mathcal{P}_k)}(\pi=i) - p^{(\mathcal{Q}_k)}(\pi=i) |  \mathbb{E}_{z\sim \pi_i}  [\ell(\mathbf{w}^k, z)] \\
        & \leq G_k\cdot \text{dist}(\mathcal{P}_k, \mathcal{Q}_k).
    \end{split}
\end{equation*}

where $ \text{dist}(\mathcal{P}_k, \mathcal{Q}_k):=  \sum_{i=1}^{I} | p^{(\mathcal{P}_k)}(\pi=i) - p^{(\mathcal{Q}_k)}(\pi=i) | $ and $\mathbb{E}_{z\sim \pi_i}  [\ell(\mathbf{w}^k, z)] \leq G_k, \forall i \in I$ due to Assumption \ref{assumption:bounded_gradient}. Recall that the constant $G_k$ clarifies the bound of loss on the corresponding model $\mathbf{w}^k$.
For the second term, with probability at least $1-\delta$, we have $|\mathcal{R}_{\mathcal{P}_k}(\mathbf{w}^k) - \mathcal{R}_{P_k}(\mathbf{w}^k)| \leq \sqrt{\frac{\log(4/\delta)}{2|P_k|}}$.
For the third term, we define $\varpi_k := \mathcal{R}_{P_k}(\mathbf{w}^{k})  - \mathcal{R}_{\mathcal{Q}_k}^*(\mathbf{w}^{\mathcal{Q}_k}) $, which can be regarded as a constant. This is because $\mathcal{R}_{P_k}(\mathbf{w}^{k}) $ represents empirical risk and $\mathcal{R}_{\mathcal{Q}_k}^*(\mathbf{w^{\mathcal{Q}_k}})$ is the ideal minimal empirical risk of model $\mathbf{w}^{\mathcal{Q}_k}$ trained on sub-distribution $\mathcal{Q}_k$.

Therefore, with probability at least $1-\delta$, given model $\mathbf{w}^{k}$, 
\begin{align*}
    \mathcal{R}_{\mathcal{Q}_k}(\mathbf{w}^k) - \mathcal{R}_{\mathcal{Q}_k}^*(\mathbf{w}^{\mathcal{Q}_k}) \leq G_k \cdot \text{dist}(\mathcal{P}_k, \mathcal{Q}_k) + \sqrt{\frac{\log(4/\delta)}{2|P_k|}} + \varpi_k.
\end{align*}
\end{proof}

\section{More experimental results}
\label{appendix:experimental_settings}

\subsection{Datasets and parameter settings}
\label{appendix:dataset_setting}
We empirically evaluate FIS on the CelebA dataset, an image dataset commonly used in the fairness literature \citep{liu2015deep}. We also evaluate FIS on two tabular datasets: UCI Adult \citep{asuncion2007uci} and Compas dataset \citep{angwin2016machine}. 

\subsubsection{CelebA dataset}
\paragraph{Dataset details}
CelebA \citep{liu2015deep} is an image dataset with 202,599 celebrity face images annotated with 40 attributes, including gender, hair colour, age, smiling, etc. The sensitive attribute is gender: $S=\{$Men, Women$\}$. We select four binary classification targets, including smiling, attractive, young, and big nose. For example, the task is to predict whether a person in an image is young ($Y=1$) or non-young ($Y=0$), among other attribute predictions.

\paragraph{Hyper-parameter details}
In all our experiments using the CelebA dataset, we train a vision transformer with patch size $(8,8)$ using SGD optimizer and a batch size of 128. 
The epochs are split into two phases: warm-up epochs (5 epochs) and training epochs (10 epochs). The default label budget per round, which represents the number of solicited data samples, is set to 256. Additionally, the default values for learning rate, momentum, and weight decay are 0.01, 0.9, and 0.0005, respectively.
We initially allocate 2\% of the training set for training purposes and the remaining 98\% for sampling. Then, we randomly select 10\% of the test data for validation. For JTT, we explore 10\% data for training purposes with weights $\lambda = 20$ for retraining misclassified examples.

\subsubsection{UCI Adult dataset}
\paragraph{Dataset details.}

The Adult dataset \citep{asuncion2007uci} predicts whether an individual's annual income falls below or exceeds $50K$, denoted as $Y=0$ and $Y=1$, respectively. This prediction is based on a variety of continuous and categorical attributes, including education level, age, gender, occupation, etc. 
The default sensitive attribute in this dataset is gender $S=\{$Men, Women$\}$ \citep{zemel2013learning}. 
In particular, we also group this dataset using age $S=\{$Teenager, Non-teenager$\}$. To achieve a balanced age distribution in the dataset, individuals with an age of less than 30 are grouped as ``Teenagers".
The dataset contains a total of 45,000 instances.
The dataset exhibits an imbalance: there are twice as many men as women, and only $15\%$ of those with high incomes are women.

\paragraph{Hyper-parameter details.} In the experiments using the Adult dataset, we train a two-layer ReLU network with a hidden size of 64.
The epochs are split into two phases: warm-up epochs (100 epochs) and training epochs (60 epochs). The default label budget per round, which represents the number of solicited data samples, is set to 1024. Additionally, the default values for learning rate, momentum, and weight decay are 0.00001, 0.9, and 0.0005, respectively.
We resample the datasets to balance the class and group membership \citep{chawla2002smote}.  The dataset is randomly split into a train and a test set in a ratio of 80 to 20. Then, we randomly re-select 20\% of the training set for initial training and the remaining 80\% for sampling. Also, 20\% examples of the test set are selected to form a validation set. 
We utilize the whole model to compute the prediction influence and fairness for examples.
Then, we randomly select 10\% of the test data for validation. For JTT, we explore 30\% data for training purposes with weights $\lambda =20$ for retraining misclassified examples.

\subsubsection{Compas dataset}
\paragraph{Dataset details.} Compas dataset, also known as the Correctional Offender Management Profiling for Alternative Sanctions dataset, is a collection of data related to criminal defendants. It contains information on approximately 6,172 individuals who were assessed for risk of re-offending. The primary task associated with this dataset is predicting whether a defendant will re-offend ($Y=1$) or not ($Y=0$) within a certain time frame after their release. 
The sensitive attribute is often considered to be race, specifically whether the individual is classified as African American or not.

\paragraph{Hyper-parameter details.}
In the experiments using the Compas dataset, we train a multi-layer neural network with one hidden layer consisting of 64 neurons.
The epochs are split into two phases: warm-up epochs (20 epochs) and training epochs (50 epochs). The default label budget per round, which represents the number of solicited data samples, is set to 128. Furthermore, the default values for learning rate, momentum, and weight decay are 0.01, 0.9, and 0.0005, respectively.
We resample the datasets to balance the class and group membership \citep{chawla2002smote}. The dataset is initially split into training and test sets at an 80-20 ratio. Then, we further split 20\% of the training set for initial training, reserving the remaining 80\% for sampling. Additionally, 20\% of the test set is selected to create a validation set. We use the entire model to calculate prediction influence and evaluate fairness for the dataset examples.

\subsection{Full version of experimental results}
\label{appendix:full_results}

\begin{table*}[H]
    \centering
    \caption{ We report the  (\texttt{test\_accuracy}, \texttt{fairness\_violation}) for evaluating the performance on the \textbf{CelebA dataset} with two binary classification targets \texttt{Smiling} and \texttt{Attractive}. We select \texttt{gender} as the sensitive attribute. }
    \vspace{-2mm}
    \begin{minipage}{.9\linewidth}
        \centering
    \vspace{1mm}
    \resizebox{1\linewidth}{!}{
    \begin{tabular}{c ccc }
    \toprule
        \multirow{2}{*}{$\epsilon = 0.05$} &   \multicolumn{3}{c}{CelebA - Smiling} \\ 
        \cmidrule{2-4}~ & (Test\_acc$\uparrow$, ~~DP$\downarrow$) & (Test\_acc$\uparrow$, ~~EOp$\downarrow$) & (Test\_acc$\uparrow$, ~~EOd$\downarrow$)\\ 
        \midrule
Base(ERM) &  (0.789 ± 0.003, 0.127 ± 0.007)& (0.789 ± 0.003, 0.067 ± 0.003)& (0.789 ± 0.003, 0.094 ± 0.008)\\   
Random & (0.837 ± 0.017, 0.133 ± 0.001)& (0.840 ± 0.020, 0.059 ± 0.005)& (0.837 ± 0.022, 0.039 ± 0.008) \\   
BALD & (0.856 ± 0.024, 0.140 ± 0.012)& (0.854 ± 0.025, 0.062 ± 0.005)& (0.856 ± 0.024, 0.033 ± 0.005)\\    
ISAL & (0.867 ± 0.022, 0.138 ± 0.010)& (0.867 ± 0.022, 0.057 ± 0.007)& (0.867 ± 0.021, 0.031 ± 0.001)\\    
JTT-20& (0.698 ± 0.018, 0.077 ± 0.006)& (0.679 ± 0.013, 0.045 ± 0.002)& (0.689 ± 0.017, 0.079 ± 0.007)\\     
 \midrule
        FIS & (0.848 ± 0.025, 0.084 ± 0.055)& (0.845 ± 0.034, 0.032 ± 0.007)& (0.800 ± 0.098, 0.033 ± 0.010)\\ 
    \bottomrule
    \end{tabular}}
    \end{minipage}%
    
    \begin{minipage}{.9\linewidth}
        \centering
    \vspace{1mm}
    \resizebox{1\linewidth}{!}{
    \begin{tabular}{c ccc }
    \toprule
        \multirow{2}{*}{$\epsilon = 0.05$} &   \multicolumn{3}{c}{CelebA - Attractive} \\ 
        \cmidrule{2-4}~ & (Test\_acc$\uparrow$, ~~DP$\downarrow$) & (Test\_acc$\uparrow$, ~~EOp$\downarrow$) & (Test\_acc$\uparrow$, ~~EOd$\downarrow$) \\
        \midrule
Base(ERM)& (0.664 ± 0.005, 0.251 ± 0.022)& (0.658 ± 0.009, 0.179 ± 0.014)& (0.664 ± 0.005, 0.181 ± 0.017) \\   
Random & (0.693 ± 0.002, 0.372 ± 0.006)& (0.693 ± 0.012, 0.250 ± 0.011)& (0.693 ± 0.002, 0.252 ± 0.008)  \\   
BALD& (0.712 ± 0.015, 0.431 ± 0.020)& (0.712 ± 0.016, 0.297 ± 0.008)& (0.712 ± 0.016, 0.307 ± 0.007) \\    
ISAL & (0.713 ± 0.009, 0.410 ± 0.003)& (0.713 ± 0.009, 0.275 ± 0.011)& (0.712 ± 0.009, 0.296 ± 0.023)\\    
JTT-20& (0.621 ± 0.015, 0.094 ± 0.004)& (0.605 ± 0.013, 0.083 ± 0.011)& (0.617 ± 0.013, 0.086 ± 0.006)\\     
 \midrule
        FIS& (0.644 ± 0.025, 0.200 ± 0.060)& (0.674 ± 0.015, 0.112 ± 0.025)& (0.677 ± 0.010, 0.159 ± 0.008)\\ 
        \bottomrule
    \end{tabular}}
    \end{minipage}
\end{table*}

\begin{table*}[!ht]
    \centering
    \caption{ We report the  (\texttt{test\_accuracy}, \texttt{fairness\_violation}) for evaluating the performance on the \textbf{CelebA dataset} with two binary classification targets \texttt{Young} and \texttt{Big Nose}. We select \texttt{gender} as the sensitive attribute.  }
    \vspace{-2mm}
    \begin{minipage}{.9\linewidth}
        \centering
    \vspace{1mm}
    \resizebox{1\linewidth}{!}{
    \begin{tabular}{c ccc }
    \toprule
        \multirow{2}{*}{$\epsilon = 0.05$} &   \multicolumn{3}{c}{\textbf{CelebA - Young}} \\ 
        \cmidrule{2-4}~ & (Test\_acc$\uparrow$, ~~DP$\downarrow$) & (Test\_acc$\uparrow$, ~~EOp$\downarrow$) & (Test\_acc$\uparrow$, ~~EOd$\downarrow$) \\ 
        \midrule
Base(ERM) &(0.755 ± 0.002, 0.190 ± 0.017) & (0.759 ± 0.005, 0.102 ± 0.005) & (0.755 ± 0.002, 0.182 ± 0.018) \\			
Random & (0.763 ± 0.008, 0.158 ± 0.016)&  (0.698 ± 0.109, 0.075 ± 0.021) & (0.766 ± 0.011, 0.156 ± 0.017) \\			
BALD& (0.776 ± 0.021, 0.165 ± 0.019) & (0.775 ± 0.020, 0.076 ± 0.007)  &(0.779 ± 0.005, 0.162 ± 0.021)  \\				
ISAL& (0.781 ± 0.020, 0.180 ± 0.014) & (0.781 ± 0.020, 0.084 ± 0.006) & (0.780 ± 0.021, 0.173 ± 0.012)	\\				
JTT-20&(0.774 ± 0.026, 0.167 ± 0.016) & (0.774 ± 0.024, 0.083 ± 0.007) & (0.772 ± 0.023, 0.171 ± 0.025) \\					
 \midrule
        FIS& (0.763 ± 0.004, 0.104 ± 0.059) & (0.773 ± 0.003, 0.041 ± 0.015)  &(0.763 ± 0.005, 0.118 ± 0.074)  \\	
    \bottomrule
    \end{tabular}}

    \end{minipage}%
    
    \begin{minipage}{.9\linewidth}
        \centering
    \vspace{1mm}
    \resizebox{1\linewidth}{!}{
    \begin{tabular}{c ccc }
    \toprule
        \multirow{2}{*}{$\epsilon = 0.05$} &   \multicolumn{3}{c}{\textbf{CelebA - Big Nose}} \\ 
        \cmidrule{2-4}~ & (Test\_acc$\uparrow$, ~~DP$\downarrow$) & (Test\_acc$\uparrow$, ~~EOp$\downarrow$) & (Test\_acc$\uparrow$, ~~EOd$\downarrow$) \\
        \midrule
Base(ERM) & (0.752 ± 0.024, 0.198 ± 0.034) & (0.755 ± 0.022, 0.206 ± 0.018) & (0.755 ± 0.022, 0.183 ± 0.029)\\			
Random & (0.760 ± 0.009, 0.177 ± 0.014) & (0.757 ± 0.004, 0.190 ± 0.029) & (0.759 ± 0.006, 0.167 ± 0.029) \\			
BALD &(0.777 ± 0.004, 0.184 ± 0.016)  &(0.765 ± 0.003, 0.209 ± 0.014) & (0.770 ± 0.004, 0.170 ± 0.015)\\				
ISAL	& (0.782 ± 0.001, 0.148 ± 0.059) & (0.782 ± 0.001, 0.154 ± 0.080) & (0.779 ± 0.006, 0.145 ± 0.065)\\				
JTT-20& (0.771 ± 0.014, 0.191 ± 0.036) & (0.758 ± 0.026, 0.223 ± 0.018) & (0.764 ± 0.019, 0.193 ± 0.016)\\					
 \midrule
        FIS& (0.779 ± 0.009, 0.089 ± 0.076) & (0.780 ± 0.013, 0.046 ± 0.072) & (0.772 ± 0.015, 0.062± 0.081)\\	
        \bottomrule
    \end{tabular}}
    \end{minipage}
\end{table*}

\begin{table*}[!ht]
    \centering
    \caption{The performance results of (\texttt{test\_accuracy}, \texttt{fairness\_violation}) on the \textbf{Adult dataset}. The sensitive attribute is \texttt{age}.  }
    \vspace{-2mm}
    
    \begin{minipage}{1\linewidth}
        \centering
    \vspace{1mm}
    \resizebox{.9\linewidth}{!}{
    \begin{tabular}{c ccc }
    \toprule
        \multirow{2}{*}{$\epsilon = 0.05$} &   \multicolumn{3}{c}{\textbf{Income (age)}} \\ 
        \cmidrule{2-4}~ & (Test\_acc$\uparrow$, ~~DP$\downarrow$) & (Test\_acc$\uparrow$, ~~EOp$\downarrow$) & (Test\_acc$\uparrow$, ~~EOd$\downarrow$) \\
        \midrule
Base(ERM) &  (0.665 ± 0.045, 0.255 ± 0.041) &(0.665 ± 0.045, 0.115 ± 0.036)& (0.665 ± 0.045, 0.158 ± 0.030)\\			
Random  &	(0.765 ± 0.021, 0.209 ± 0.042) &(0.758 ± 0.027, 0.127 ± 0.013) &(0.764 ± 0.018, 0.133 ± 0.027)\\			
BALD&	(0.767 ± 0.019, 0.203 ± 0.017)& (0.703 ± 0.111, 0.117 ± 0.013) &(0.763 ± 0.022, 0.128 ± 0.014)\\				
ISAL	 &	(0.765 ± 0.020, 0.215 ± 0.011)& (0.755 ± 0.028, 0.128 ± 0.013) &(0.761 ± 0.024, 0.138 ± 0.009)\\				
JTT-20 &(0.751 ± 0.013, 0.262 ± 0.020)& (0.742 ± 0.018, 0.149 ± 0.021) &(0.745 ± 0.014, 0.171 ± 0.012)\\					
 \midrule
        FIS& (0.766 ± 0.013, 0.214 ± 0.009)& (0.757 ± 0.034, 0.113 ± 0.017)& (0.763 ± 0.011, 0.143 ± 0.023)\\	
        \bottomrule
    \end{tabular}}
    \end{minipage}
\end{table*}

\begin{table*}[!ht]
    \centering
    \caption{The performance results of (\texttt{test\_accuracy}, \texttt{fairness\_violation}) on the \textbf{Compas dataset}. The selected sensitive attribute is \texttt{race}.}
    \vspace{-2mm}
    \begin{minipage}{.9\linewidth}
        \centering
    \vspace{1mm}
    \resizebox{1\linewidth}{!}{
    \begin{tabular}{c ccc }
    \toprule
        \multirow{2}{*}{$\epsilon = 0.05$} &   \multicolumn{3}{c}{\textbf{Recidivism}} \\ 
        \cmidrule{2-4}~ & (Test\_acc$\uparrow$, ~~DP$\downarrow$) & (Test\_acc$\uparrow$, ~~EOp$\downarrow$) & (Test\_acc$\uparrow$, ~~EOd$\downarrow$) \\ 
    \midrule
   Base(ERM)       & (0.675 ± 0.005, 0.333 ± 0.008) &   (0.675 ± 0.005, 0.267 ± 0.010)  &  (0.675 ± 0.005, 0.284 ± 0.010)\\
    Random    & (0.689 ± 0.007, 0.305 ± 0.023)  &  (0.686 ± 0.016, 0.253 ± 0.035)  &  (0.688 ± 0.006, 0.256 ± 0.023)\\
    BALD      & (0.688 ± 0.011, 0.313 ± 0.012)  &  (0.686 ± 0.015, 0.256 ± 0.031)   & (0.688 ± 0.011, 0.263 ± 0.011)\\
    ISAL     &  (0.697 ± 0.002, 0.308 ± 0.025)  &  (0.698 ± 0.004, 0.274 ± 0.022)   & (0.697 ± 0.001, 0.260 ± 0.026)\\
    JTT-20   & (0.646 ± 0.009, 0.240 ± 0.016)    &(0.630 ± 0.024, 0.141 ± 0.028)   & (0.646 ± 0.009, 0.200 ± 0.007)\\
    \midrule
    FIS     &   (0.690 ± 0.002, 0.299 ± 0.029)  &  (0.694 ± 0.002, 0.241 ± 0.035)  &  (0.698 ± 0.005, 0.252 ± 0.030) \\
    \bottomrule
    \end{tabular}}
    \end{minipage}%
    
\end{table*}

\subsection{Exploring the impact of label budgets}
\label{appendix:impact_of_label_budget}

 {In our study, we examine how varying label budgets $r$ influence the balance between accuracy and fairness. We present the results of test accuracy and fairness disparity across different label budgets on the CelebA, Compas, and Adult datasets. In these experiments, we use the demographics parity (DP) as our fairness metric. For convenience, we maintain a fixed label budget per round, using rounds of label budget allocation to demonstrate its impact. The designated label budgets per round for the CelebA, Compas, and Adult are 256, 128, and 512, respectively.
In the following figures, the $x$-axis is both the number of label budget rounds. The $y$-axis for the left and right sub-figures are test accuracy and DP gap, respectively.
As observed in Figures \ref{fig:celeba_impact_label_budget}-\ref{fig:compas_impact_label_budget}, compared to the three baselines (BALD, JTT-20, and ISAL), our approach substantially reduces the DP gap without sacrificing test accuracy.}

Specifically, on the Adult dataset, both accuracy and fairness violation converge to similar numerical values when the budget is lower than 20, suggesting a potential overfitting of the model to insufficient training examples. With a larger budget, our algorithms outperform other baseline methods, achieving higher accuracy and lower demographic disparity.

\begin{figure}[h]
    \centering
    \includegraphics[width=0.9\linewidth]{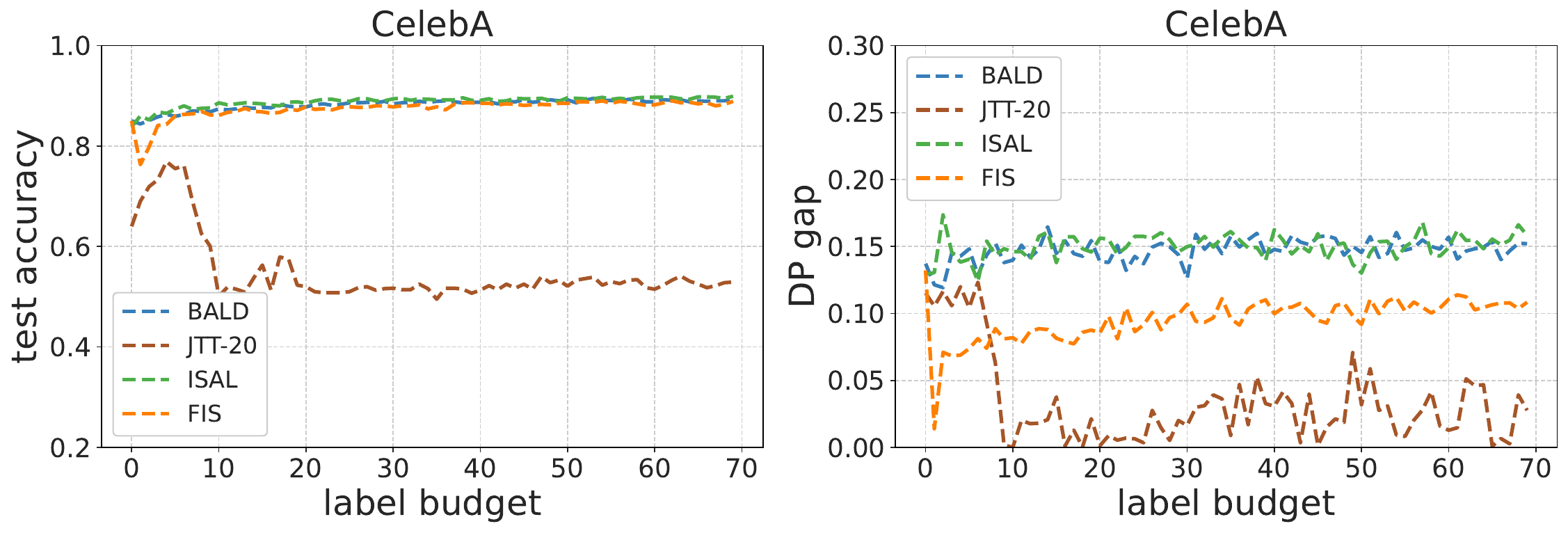}
    \caption{ The impact of label budgets on the test accuracy \& DP gap in the CelebA dataset. The binary classification targets is \texttt{Smiling}.}
    \label{fig:celeba_impact_label_budget}
\end{figure}

\begin{figure}[h]
    \centering
    \includegraphics[width=0.9\linewidth]{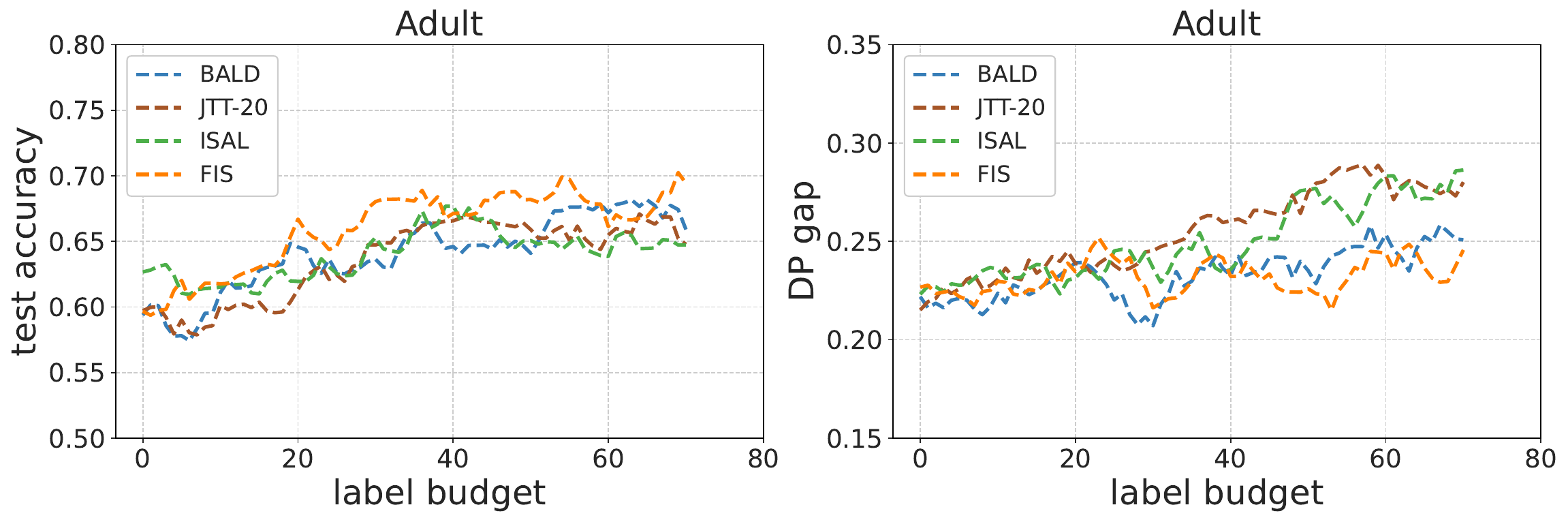}
    \caption{ The impact of label budgets on the test accuracy \& DP gap in the Adult dataset. The sensitive attribute is \texttt{sex}.}
    \label{fig:adult_impact_label_budget}
\end{figure}

\begin{figure}[h]
    \centering
    \includegraphics[width=0.9\linewidth]{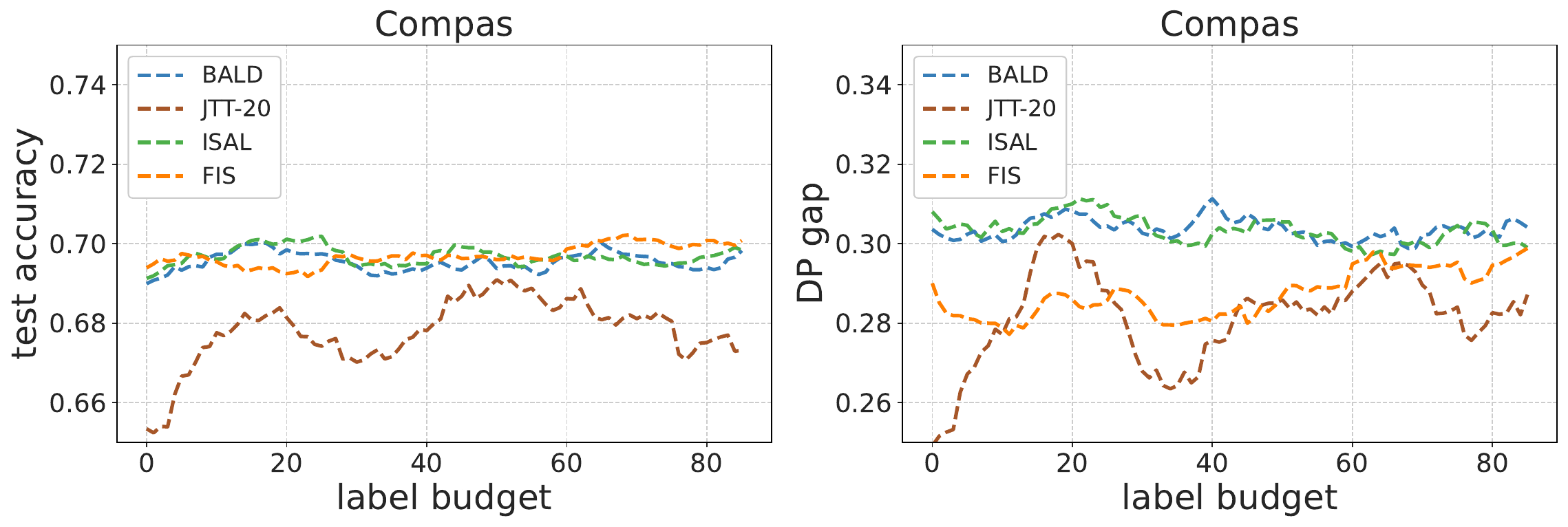}
    \caption{  The impact of label budgets on the test accuracy \& DP gap in the Compas dataset. The sensitive attribute is \texttt{race}.}
    \label{fig:compas_impact_label_budget}
\end{figure}

\subsection{The role of validation dataset size} \label{appendix:impact_of_validaion_size}

In this subsection, we explore the impact of adjusting the validation set size on our algorithm's performance. We present the test accuracy and fairness disparity across different validation set sizes on the CelebA, Compas, and Adult datasets. 
Note that the default validation set size for image and tabular datasets is set to 1\% and 4\% of the whole dataset size, respectively. That is, the default validation set sizes are 1996 (CelebA), 1800 (Adult), and 247 (Compas) instances, respectively.
In particular, given the smaller size of the default validation set, the minimum scale of the validation set size is set to $\sfrac{1}{5}\times$ (nearly 400 CelebA images). Tables~\ref{table:celeba_validation_set_size} and \ref{table:tabular_validate_set_size} present the performance results on the CelebA, UCI Adult, and Compas datasets, respectively.

\begin{table*}[!htbp]
    \centering
        \caption{The performance results of (\texttt{test\_accuracy}, \texttt{fairness\_violation}) on the \textbf{CelebA dataset} when the validation set size is reduced to $\sfrac{1}{2}\times$ and $\sfrac{1}{5}\times$. Our algorithm retains the test accuracy and fairness violation when we vary the validation set size.}
    \resizebox{\linewidth}{!}{
    \begin{tabular}{c   ccc ccc}
    \toprule
        \multirow{2}{*}{$\epsilon = 0.05$} &  \multicolumn{3}{c}{\textbf{CelebA - Smiling}}  &  \multicolumn{3}{c}{\textbf{CelebA - Attractive}} \\
        \cmidrule{2-4} \cmidrule(lr){5-7}
        ~ & (Test\_acc$\uparrow$, DP$\downarrow$) & (Test\_acc$\uparrow$, EOp$\downarrow$) & (Test\_acc$\uparrow$, EOd$\downarrow$)  & (Test\_acc$\uparrow$, DP$\downarrow$) & (Test\_acc$\uparrow$, EOp$\downarrow$) & (Test\_acc$\uparrow$, EOd$\downarrow$)\\ 
    \midrule
        $1 \times$ &  (0.848, 0.084) & (0.876, 0.031) & (0.864, 0.030) &  (0.680, {0.285}) & 	(0.695, 0.148)  & (0.692, 0.148)  \\
        $\sfrac{1}{2}\times$ & (0.872, 0.105) & (0.891, 0.042) & (0.880, 0.028) & (0.648, 0.249) & (0.688, 0.188)  & (0.678, 0.147)\\
        $\sfrac{1}{5}\times$ & (0.872, 0.117) & (0.863, 0.057) & (0.886, 0.028)& (0.604, 0.171) & (0.707, 0.209) & (0.645, 0.145)\\
    \bottomrule
    \toprule
        \multirow{2}{*}{$\epsilon = 0.05$} &  \multicolumn{3}{c}{\textbf{CelebA - Young}}  &  \multicolumn{3}{c}{\textbf{CelebA - Big\_Nose}} \\
        \cmidrule{2-4} \cmidrule(lr){5-7}
        ~ & (Test\_acc$\uparrow$, DP$\downarrow$) & (Test\_acc$\uparrow$, EOp$\downarrow$) & (Test\_acc$\uparrow$, EOd$\downarrow$)  & (Test\_acc$\uparrow$, DP$\downarrow$) & (Test\_acc$\uparrow$, EOp$\downarrow$) & (Test\_acc$\uparrow$, EOd$\downarrow$)\\
    \midrule
    $1\times$ &  (0.766, {0.139})	&	(0.775, {0.043}) 	&(0.769, {0.168}) 	&(0.771, 0.156) &(0.765, {0.129})	& (0.758, {0.155})\\
    $\sfrac{1}{2}\times$ & (0.735, 0.093) & (0.762, 0.067) & (0.769, 0.055)& (0.771, 0.054)  &(0.761, 0.162) & (0.748, 0.096)\\
    $\sfrac{1}{5} \times$ & (0.743, 0.107)  &(0.780, 0.097) & (0.757, 0.166) & (0.772, 0.095)  &(0.750, 0.300)  &(0.760, 0.156)\\
    \bottomrule
    \end{tabular}}
    \label{table:celeba_validation_set_size}
\end{table*}

\begin{table*}[!htbp]
    \centering
        \caption{We examine the performance results of (\texttt{test\_accuracy}, \texttt{fairness\_violation}) on the tabular datasets  (\textbf{Left}: Adult; \textbf{Right}: Compas) when the validation set size is reduced to $\sfrac{1}{2}\times$, $\sfrac{1}{4}\times$, and $\sfrac{1}{20}\times$. We observe that our algorithm still retains the test accuracy and fairness violation when we vary the validation set size. }
    \resizebox{\linewidth}{!}{
    \begin{tabular}{c   ccc ccc}
    \toprule
        \multirow{2}{*}{$\epsilon = 0.05$} &  \multicolumn{3}{c}{\textbf{Adult - Income(Age)}}  &  \multicolumn{3}{c}{\textbf{Compas - Recidivism}} \\
        \cmidrule{2-4} \cmidrule(lr){5-7}
        ~& (Test\_acc$\uparrow$, DP$\downarrow$) & (Test\_acc$\uparrow$, EOp$\downarrow$) & (Test\_acc$\uparrow$, EOd$\downarrow$)  & (Test\_acc$\uparrow$, DP$\downarrow$) & (Test\_acc$\uparrow$, EOp$\downarrow$) & (Test\_acc$\uparrow$, EOd$\downarrow$) \\ 
        \midrule
        $1 \times$ & (0.757, 0.198) &  (0.718, 0.124) &  (0.750, 0.125) &  (0.690, 0.313)  &  (0.696, 0.249)  &  (0.702, 0.257) \\
        $\sfrac{1}{2}\times$ & (0.717, 0.259) &  (0.634, 0.123) &  (0.736, 0.143) & (0.683, 0.270)  &  (0.680, 0.247)  &  (0.693, 0.244)\\
        $\sfrac{1}{4}\times$ & (0.749, 0.196) &  (0.750, 0.121) &  (0.747, 0.137)  & (0.677, 0.283)  &  (0.682, 0.276)   & (0.680, 0.244)\\
        $\sfrac{1}{20}\times$ & (0.721, 0.205) &  (0.706, 0.148) &  (0.706, 0.179) & (0.689, 0.289)  &  (0.668, 0.236)  &  (0.683, 0.252)\\	
    \bottomrule
    \end{tabular}}
    \label{table:tabular_validate_set_size}
\end{table*}

\subsection{Benchmarking model performance with validation set enhancements}\label{apx:validation_set_enhancements}

Note that we resort to an additional validation set for developing FIS. 
 To demonstrate FIS's advantage at the same levels of information, we introduce a new baseline called \texttt{Random+Val}.  This method involves continuing to train the model with a randomly sampled validation set. Specifically, we start with the \texttt{Random}'s last saved checkpoint and train it further using the validation set. In particular, we would incorporate a fairness regularizer with dynamic weight to train the model using validation data, considering its sensitive attributes to reduce fairness disparity.
 Due to its small size, we limit training to 10 epochs to avoid overfitting. The performance results of \texttt{Random}, \texttt{Random+Val}, and FIS are given in Table \ref{table:validate_size_to_train}.

\begin{table*}[!htbp]
    \centering
        \caption{Comparative analysis of (\texttt{test\_accuracy}, \texttt{fairness\_violation}) in the CelebA, Adult and Compas datasets. The table illustrates that even at the same information level (using the validation set to train), FIS can obtain better performances. Similarly, we highlight all the fairer but without sacrificing accuracy results in boldface compared to \texttt{Random}.}
    \resizebox{\linewidth}{!}{
    \begin{tabular}{c   ccc ccc}
    \toprule
        \multirow{2}{*}{$\epsilon = 0.05$} &  \multicolumn{3}{c}{\textbf{CelebA - Smiling}}  &  \multicolumn{3}{c}{\textbf{CelebA - Attractive}} \\
        \cmidrule{2-4} \cmidrule(lr){5-7}
        ~ & (Test\_acc$\uparrow$, DP$\downarrow$) & (Test\_acc$\uparrow$, EOp$\downarrow$) & (Test\_acc$\uparrow$, EOd$\downarrow$)  & (Test\_acc$\uparrow$, DP$\downarrow$) & (Test\_acc$\uparrow$, EOp$\downarrow$) & (Test\_acc$\uparrow$, EOd$\downarrow$)\\ 
    \midrule
        Random  & (0.853, 0.132)	&	(0.863, 0.053)	& (0.861, 0.031) & (0.696, 0.367)&	(0.708, 0.253)	&(0.696, 0.243)\\
        Random + Val & (0.801, 0,115) & (0.872, 0.139) & (0.879, 0.153) & (0.638, 0.199) & (0.699, 0.366) & (0.699, 0.355)\\
        \midrule
        FIS &  \textbf{(0.877, 0.122)}	&	\textbf{(0.886, {0.040}) }& \textbf{(0.882, {0.023})}	& \textbf{(0.680, {0.285})}	& \textbf{(0.695, 0.148)} & \textbf{(0.692, 0.148)}  \\
    \bottomrule
    \toprule
        \multirow{2}{*}{$\epsilon = 0.05$} &  \multicolumn{3}{c}{\textbf{Adult - Income(Age)}}  &  \multicolumn{3}{c}{\textbf{Compas - Recidivism}} \\
        \cmidrule{2-4} \cmidrule(lr){5-7}
        ~& (Test\_acc$\uparrow$, DP$\downarrow$) & (Test\_acc$\uparrow$, EOp$\downarrow$) & (Test\_acc$\uparrow$, EOd$\downarrow$)  & (Test\_acc$\uparrow$, DP$\downarrow$) & (Test\_acc$\uparrow$, EOp$\downarrow$) & (Test\_acc$\uparrow$, EOd$\downarrow$) \\ 
        \midrule
        Random &    (0.745, 0.236)  &(0.729, 0.136) & (0.748, 0.151)	 &  (0.696, 0.316) & (0.698, 0.258) & (0.694, 0.269)		 \\
        Random + Val & \textbf{(0.762, 0.161)} & (0.743, 0.266) & \textbf{(0.788, 0.131) } & (0.604, 0.136) & (0.623, 0.147) & (0.628, 0.153)\\
        \midrule
        FIS &   \textbf{(0.751, 0.205)}  &\textbf{(0.725, 0.130)}  & \textbf{(0.750, 0.125)}	&  \textbf{(0.688, 0.316)}   &\textbf{(0.695, 0.250)}  & \textbf{(0.693, 0.254)}		\\
    \bottomrule
    \end{tabular}}
    \label{table:validate_size_to_train}
\end{table*}

\section{Potential privacy leakage from non-demographic and demographic feature correlations}
\label{appendix:correlation_non-demongraphic_and_demographic}

Note that some non-demographic features may correlate with sensitive attributes (demographic information), which could lead to potential privacy leakage issues \citep{zhao2022towards}. However, it is crucial to clarify that our method specifically limits its query to the data. We only query the true labels for a selected subset of unlabeled examples. While correlations between non-demographic variables and sensitive attributes may exist in the underlying data, our method itself does not introduce any additional privacy leakage beyond what is inherent in the original dataset. On the other hand, our work's primary focus is not on addressing privacy concerns related to demographic information. Rather, our main objective is to reduce fairness disparities while maintaining a favorable utility trade-off. We achieve this without explicitly incorporating any additional sensitive information into the process and without directly engaging with the complex privacy implications of demographic data usage.

To address this potential privacy concern, we can analyze it via differential privacy. Consider a function $\phi(\cdot)$ that maps non-demographic information (feature $\boldsymbol{x}$ and label $y$) to a sensitive attribute $s$. Due to insufficient information, $\phi(\cdot)$ is not deterministic, making it unable to precisely estimate $s$. If a deterministic mapping function existed, it would unavoidably lead to privacy leakage. Therefore, we assume that given the same features and label, the function $\phi(\cdot)$ output might vary.
For example, in the Adult dataset, $\boldsymbol{x}$ includes age and credit history, $y$ is the income, and $s$ is the gender.  In practice, both men and women have a probability of earning more than 50K (i.e., $y=1$) given the same age, credit history, etc (same feature $\boldsymbol{x}$).
Suppose $P(\phi(\boldsymbol{x}, y)=s \mid S =s, \boldsymbol{x}, y) \leq 1- \epsilon_0$ and $P(\phi(\boldsymbol{x}, y)=s \mid S =s', \boldsymbol{x}, y) \geq \epsilon_1$ for all $\boldsymbol{x}, y, s, s'$ (where $s \neq s'$). Here, $S$ represents the sensitive attribute variable, which is unknown to $\phi(\cdot)$.
Then, we have
$$
\frac{\mathbb{P}(\phi(\boldsymbol{x}, y)=\tilde{s} | S = s, \boldsymbol{x}, y)}{\mathbb{P}(\phi(\boldsymbol{x}, y)=\tilde{s} | S = s', \boldsymbol{x}, y)}  \leq \frac{\max \mathbb{P}(\phi(\boldsymbol{x}, y)=\tilde{s} \mid S=s, \boldsymbol{x}, y)}{\min \mathbb{P}(\phi(\boldsymbol{x}, y)=\tilde{s} \mid S=s^{\prime}, \boldsymbol{x},y)} \leq \frac{1-\epsilon_0}{\epsilon_1}=e^{\varepsilon} .
$$
where $\varepsilon=\ln \left(\frac{1-\epsilon_0}{\epsilon_1}\right)$. In practice, if the mapping function is too strong, i.e. $\ln \left(\frac{1-\epsilon_0}{\epsilon_1}\right)$ is too large, we can add additional noise to reduce their informativeness and therefore better protect privacy.

\end{document}